%% file: large_ss.tex
\documentclass{article}
\PassOptionsToPackage{numbers}{natbib}
\usepackage[final]{neurips_2025}

\title{
  Any-stepsize Gradient Descent for Separable Data \\ under Fenchel--Young Losses
}

\usepackage[utf8]{inputenc} % allow utf-8 input
\usepackage[T1]{fontenc}    % use 8-bit T1 fonts
\usepackage{hyperref}       % hyperlinks
\usepackage{url}            % simple URL typesetting
\usepackage{booktabs}       % professional-quality tables
\usepackage{amsfonts}       % blackboard math symbols
\usepackage{nicefrac}       % compact symbols for 1/2, etc.
\usepackage{microtype}      % microtypography
\usepackage{xcolor}         % colors

\usepackage{times}

\input{preamble}

\allowdisplaybreaks[4]
\renewcommand{\proofname}{Proof.}

\author{%
 Han Bao\thanks{This work was primarily conducted during the period when HB was affiliated with Kyoto University, and SS with the University of Tokyo and RIKEN AIP.} \\ The Institute of Statistical Mathematics \\ \texttt{bao.han@ism.ac.jp} \\
 \And
 Shinsaku Sakaue\footnotemark[1] \\ CyberAgent \\ \texttt{shinsaku.sakaue@gmail.com} \\
 \AND
 Yuki Takezawa \\ Kyoto University and OIST \\ \texttt{takezawa@ml.ist.i.kyoto-u.ac.jp}
}

\newlist{assumpenum}{enumerate}{1}
\setlist[assumpenum]{label={\Alph*.},ref={\theassumption\Alph*}}
\crefname{assumpenumi}{Assumption}{Assumptions}
\crefname{assumption}{Assumption}{Assumptions}
\crefname{corollary}{Corollary}{Corollaries}
\crefname{definition}{Definition}{Definitions}
\crefname{lemma}{Lemma}{Lemmas}
\crefname{proposition}{Proposition}{Propositions}
\crefname{figure}{Figure}{Figures}

\begin{document}

\maketitle

\setcounter{footnote}{0}

\begin{abstract}%
  The gradient descent (GD) has been one of the most common optimizer in machine learning.
  In particular, the loss landscape of a neural network is typically sharpened during the initial phase of training, making the training dynamics hover on the edge of stability.
  This is beyond our standard understanding of GD convergence in the stable regime where stepsize is chosen sufficiently smaller.
  Recently, \citet{Wu2024COLT} have shown that GD converges with much larger stepsize under linearly separable logistic regression.
  Although their analysis hinges on the self-bounding property of the logistic loss, which seems to be a cornerstone to establish a modified descent lemma, our pilot study shows that other loss functions without the self-bounding property can make GD attain arbitrarily small loss with large stepsize.
  To further understand what property of a loss function matters in GD, we aim to show large-stepsize GD convergence for a general loss function based on the framework of \emph{Fenchel--Young losses}.
  We essentially leverage the classical perceptron argument to derive the iteration complexity for achieving $\epsilon$-optimal loss, which is possible for a majority of Fenchel--Young losses.
  This convergence result highlights that the self-bounding property may not be necessary for GD to attain arbitrarily small loss.
  Moreover, when a loss function entails \emph{separation margin}, a notion relevant to the margin in support vector machines, GD often yields faster convergence than typical GD rate $T=\Omega(\epsilon^{-1})$ for convex smooth objectives.
  Specifically, GD with the Tsallis entropy attains $\epsilon$-optimal loss with the rate $T=\Omega(\epsilon^{-1/2})$,
  and the R{\'e}nyi entropy achieves the far better rate $T=\Omega(\epsilon^{-1/3})$.
\end{abstract}

\section{Introduction}
\label{section:introduction}
Gradient-based optimizers are prevalent in the modern machine learning community with deep learning thanks to its scalability and plasticity.
Among many variants, GD remains to be a standard choice.
GD with constant stepsize is written as follows:
\begin{equation}
  \label{equation:gd} \tag{GD}
  \wbf_{t+1} \defeq \wbf_t - \eta\nabla L(\wbf_t), \qquad \text{for $t=0,1,\dots,T-1$,}
\end{equation}
where $\wbf\in\Rbb^d$ is the optimization variables, $L(\cdot)$ is the loss function, and $\eta>0$ is stepsize fixed across all steps.
The \emph{descent lemma}~\cite[Section~1.2.3]{Nesterov2018} is a key to GD convergence: for $\beta$-smooth objective $L$, the stepsize choice $\eta<2/\beta$ ensures that $L(\wbf_t)$ monotonically decreases.
Nonetheless, little optimization theory has been known beyond the threshold $\eta>2/\beta$; though modern neural networks exhibit much smaller smoothness values than practically used stepsize values~\cite{Yao2018NeurIPS,Tsuzuku2020ICML}.
Moreover, recent studies have reported that GD trajectories of neural networks tend to inflate the sharpness of the loss landscape and hover on the \emph{edge of stability} (EoS) before convergence~\cite{Lewkowycz2020,Cohen2021ICLR,Ahn2022ICML}.

\begin{figure}
  \centering
  \includegraphics[width=0.4\textwidth]{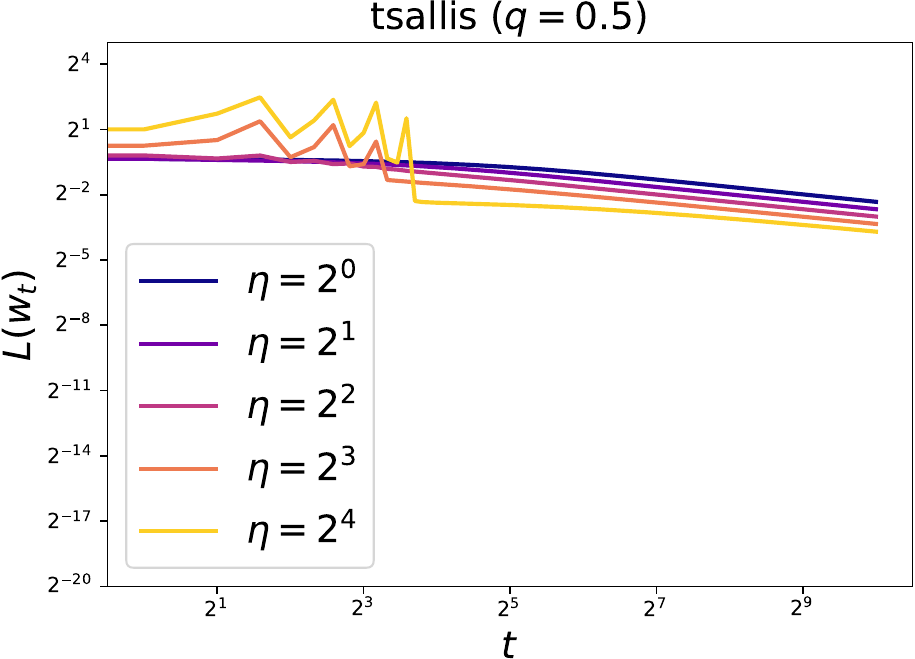} \hspace{10pt}
  \includegraphics[width=0.4\textwidth]{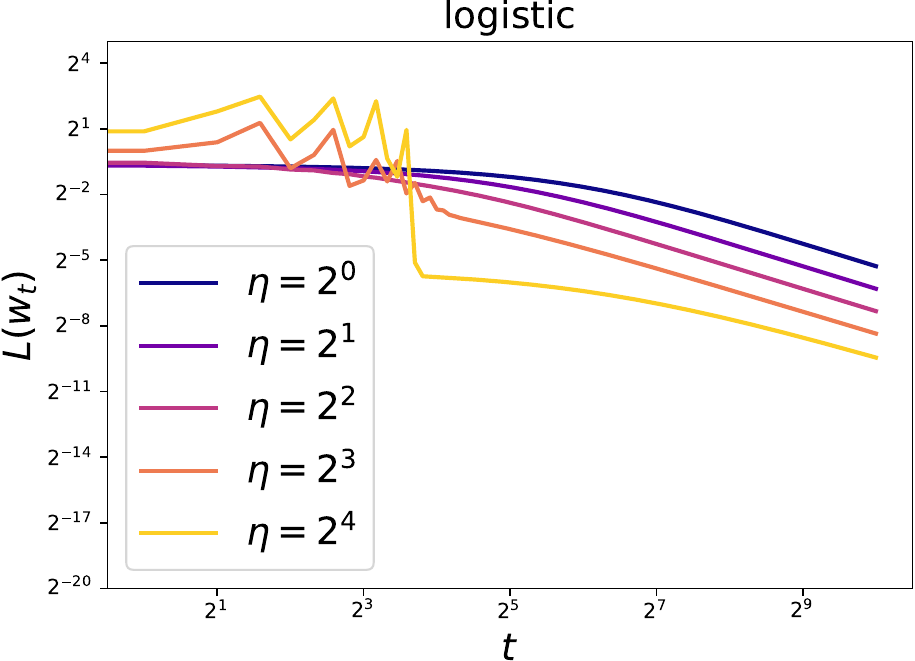} \\
  \includegraphics[width=0.4\textwidth]{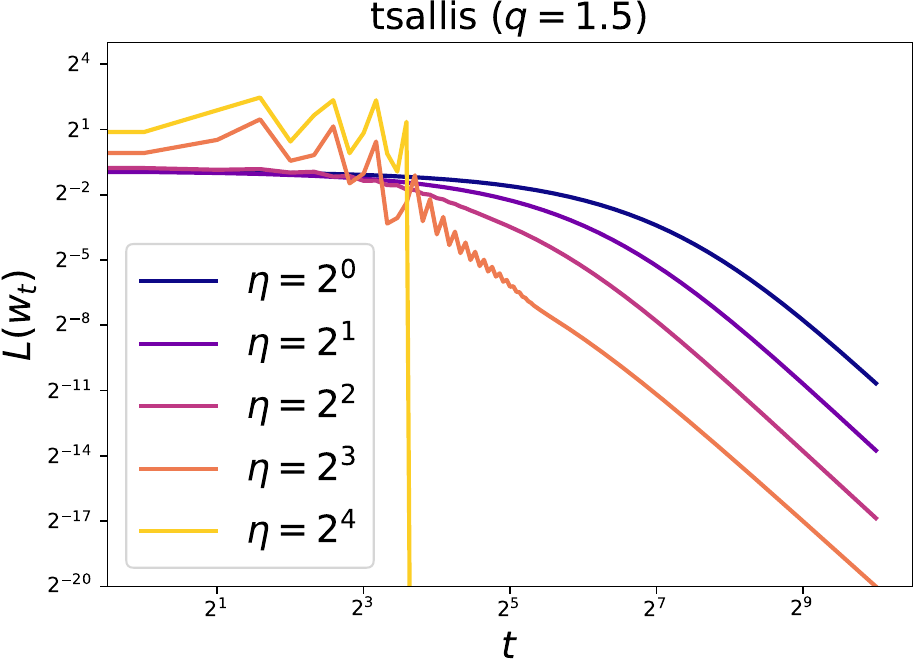} \hspace{10pt}
  \includegraphics[width=0.4\textwidth]{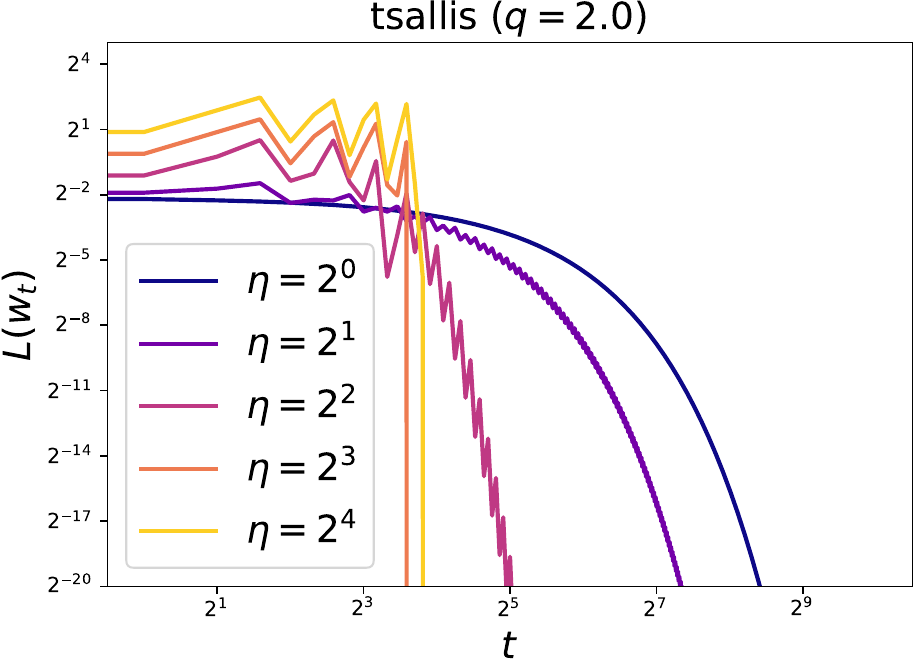}
  \caption{
    Pilot studies of GD with the same toy dataset as \cite{Wu2024COLT}.
    The dataset consists of four points, $\xbf_1=[1,0.2]^\top$, $y_1=1$, $\xbf_2=[-2,0.2]^\top$, $y_2=1$, $\xbf_3=[-1,-0.2]^\top$, $y_3=-1$, $\xbf_4=[2,-0.2]^\top$, and $y_4=-1$.
    GD is run with initialization $\wbf_0=[0,0]^\top$.
    Note that the logistic loss corresponds to the Tsallis $1$-loss.
    The Tsallis $2$- and $q$-loss are also known as the modified Huber loss~\cite{Zhang2004ICML} and $q$-entmax loss~\cite{Peters2019ACL}, respectively.
  }
  \label{figure:pilot}
\end{figure}

Among several recent developments in the theory of large-stepsize GD (which we will review in \cref{section:related}), \citet{Wu2024COLT} investigated the large-stepsize behavior of GD by using the binary logistic regression with a linearly separable data, a minimal synthetic setting.
They showed that GD initially oscillates with non-monotonic loss values (the EoS phase), which terminates in finite time (phase transition), and then the loss value decreases monotonically (the stable phase).
Beyond the logistic loss, these results have been extended to loss functions with the \emph{self-bounding property}: for a differentiable loss function $\ell\colon\Rbb\to\Rbb$ and its absolute derivative $g(\cdot)\defeq|\ell'(\cdot)|$, $\ell$ satisfies
\begin{equation}
  \label{equation:self_bounding_property}
  \ell(z)\le\ell(x)+\ell'(x)(z-x)+C_\beta g(x)(z-x)^2
  \quad \text{$\forall z, x$ with $|z-x|<1$,} \quad \text{for some $C_\beta>0$.}
\end{equation}
The self-bounding property generalizes the polynomially-tailed loss~\cite{Ji2020COLT,Ji2021ALT}, and refines the standard smoothness property by allowing the smoothness modulus locally adaptive to the derivative, such that $C_\beta g(x)$.
Thus, large $\eta$ can be cancelled out with the vanishingly small loss gradient after the phase transition~\cite[Lemma~29]{Wu2024COLT}, and GD follows the descent direction.

In this paper, we study GD with large stepsize under a wide range of loss functions to identify a key factor to induce the convergent behavior.
This is motivated by our pilot study shown in \cref{figure:pilot}, where we found that GD with large stepsize such as $\eta=2^4$ remains to converge under the Tsallis $q$-loss (detailed in \cref{section:examples}), even if the stepsize has gone beyond the classical stable regime.
It is noteworthy therein that the Tsallis $q$-loss with $q>1$ does not enjoy the self-bounding property.
How much does the self-bounding property play a vital role in large-stepsize GD convergence?
We specifically consider \emph{Fenchel--Young losses}~\cite{Blondel2020JMLR}, a class of convex loss functions generated by a potential function $\phi$, as a template of loss functions.
Fenchel--Young losses have been used in applications such as structured prediction~\cite{Niculae2018ICML}, differentiable programming~\cite{Berthet2020NeurIPS}, and model selection~\cite{Bao2021AISTATS}, while being used as a theoretical tool for online learning~\cite{Sakaue2024COLT,Sakaue2025}.
We identify that Fenchel--Young losses with \emph{separation margin} (formally introduced in \cref{section:preliminary}), a relevant notion to the margin in support vector machines,
can often benefit from better GD convergence rates.
We say a loss function has separation margin if the loss value vanishes with a sufficiently large positive prediction margin.
Specifically, our main result is informally stated as follows.
\begin{theorem}[{Informal version of \cref{theorem:gd}}]
  \label{theorem:main}
  Consider a binary classification dataset that is linearly separable.
  We run \eqref{equation:gd} with arbitrary constant stepsize $\eta>0$ and initialization $\wbf_0=\zerobf$ under a Fenchel--Young loss generated by twice continuously differentiable and convex potential~$\phi$ with separation margin.
  For $\epsilon>0$, after at most $T$ steps of \eqref{equation:gd}, where
  \[
    T=\Omega(\epsilon^{-\alpha}) \quad \text{and} \quad \alpha=\limsup_{\mu\downarrow0}\frac{\phi'(\mu)}{\mu\phi''(\mu)}\left[1-\frac{\phi(\mu)}{\mu\phi'(\mu)}\right],
  \]
  we have $L(\wbf_T)\le\epsilon$.%
  \footnote{
    Throughout this paper, we consider $\eta=\Theta(1)$ with respect to the error tolerance $\epsilon$ when we say arbitrary stepsize
    unless otherwise noted.
  }
\end{theorem}
As defined in \cref{section:preliminary}, a loss function with separation margin vanishes for a sufficiently large prediction margin,
which is a natural indicator of correct classification used in support vector machines.
The order of the convergence rate~$T=\Omega(\epsilon^{-\alpha})$ differs across various potential~$\phi$.
With a specific choice, the rate can be~$T=\Omega(\epsilon^{-1/2})$ (with~$\phi$ being the Tsallis $2$-entropy)
and~$T=\Omega(\epsilon^{-1/3})$ (with~$\phi$ being the R{\'e}nyi~$2$-entropy, also known as the collision entropy~\cite{Bosyk2012}).
Remarkably, these convergence rates are better than the classical GD convergence rate~$T=\Omega(\epsilon^{-1})$ under the stable regime,
and even better than the convergence rate of the logistic loss after undergoing the EoS and phase transition~\cite{Wu2024COLT}.
Both the Tsallis and R{\'e}nyi entropies above lack the self-bounding property but have separation margin.
Therefore, we advocate the importance of separation margin for better GD convergence rates.
We compare different Fenchel--Young losses in \cref{section:examples} and contrast our convergence result with the EoS and implicit bias in \cref{section:discussion}.

We present \cref{theorem:main} formally in \cref{section:main}.
Our proof leverages the classical perceptron argument~\cite{Novikoff1962} without relying on the descent lemma at all.
Intuitively speaking, we track the growth of the parameter alignment $\inpr{\wbf_t}{\wbf_*}$ with the optimal separator $\wbf_*$.
When a loss entails separation margin, $\inpr{\wbf_t}{\wbf_*}$ cannot grow arbitrarily large (as we simulate in \cref{figure:norm} later)
while each step of \eqref{equation:gd} improves a lower bound on $\inpr{\wbf_t}{\wbf_*}$, leading to the convergence.
\Cref{section:proof_sketch} describes this proof overview in detail.
This is different from the proof of \citet{Wu2024COLT}, whose core is the modified decent lemma (recapped in \cref{lemma:stable_phase} in the appendix) based on the self-bounding property.
Although the perceptron argument is partially used therein \cite{Wu2024COLT}, the average loss is finally controlled by the modified descent lemma, and thus the proof is only applicable to the self-bounding losses.

\subsection{Related work}
\label{section:related}
Gradient descent with large stepsize has attracted attention recently.
Specifically, non-monotonic behaviors of loss functions~\cite{Xing2018} and the sharpness adaptivity to loss landscapes~\cite{Lewkowycz2020,Cohen2021ICLR} have been observed empirically.
It was argued that the sharpness tends to initially increases until the classical stable regime breaks down, and hovers on this boundary, termed as the edge of stability~\cite{Cohen2021ICLR}.
This observation mainly sparks two questions: why the loss landscape hovers on the EoS, and why converging.
Answering either question must go beyond the classical optimization theory under the stable regime.

On why hovering on the EoS, let us make a brief review, though it is not a central focus of this paper: \citet{Ahn2022ICML} is a seminal work to empirically investigate the homogeneity of loss functions contributes to maintain the EoS.
Later, it was showed that normalized GD (represented by scale-invariant losses) adaptively leads their intrinsic stepsize toward sharpness reduction~\cite{Lyu2022NeurIPS}.
The sharpness fluctuation is often attributed to the non-negligible third-order Taylor remainder of the loss landscape~\cite{Ma2022,Damian2023ICLR}.

We rather focus on why GD attains arbitrarily small loss with much larger stepsize.
In this line, previous studies show convergence based on specific models such as multi-scale loss function~\cite{Kong2020NeurIPS}, quadratic functions~\cite{Arora2022ICML}, matrix factorization~\cite{Wang2022ICLR,Chen2023ICML}, a scalar multiplicative model~\cite{Zhu2023ICLR,Kreisler2023ICML}, a sparse coding model~\cite{Ahn2023NeurIPS}, and linear logistic regression~\cite{Wu2023NeurIPS}.
Among them, we advocate the logistic regression setup proposed by~\citet{Wu2023NeurIPS} because it is relevant to implicit bias of GD~\cite{Soudry2018,Ji2019COLT,Ravi2024NeurIPS}, and moreover, \citet{Wu2024COLT} corroborates the benefit of large stepsize in GD convergence rate.
Our work is provoked by~\citet{Wu2024COLT}, questioning what structure in a loss function leads GD to arbitrarily small loss.
Indeed, we observe in \cref{figure:pilot} that loss functions without the self-bounding property~\eqref{equation:self_bounding_property} can make GD attain arbitrarily small loss, though the self-bounding property seems essential to calm the EoS down to the stable phase~\cite{Wu2024COLT} as well as to establish the max-margin directional convergence~\cite{Ji2019COLT,Ravi2024NeurIPS}. 
A similar question to ours is raised by~\citet{Tyurin2024}, who argues that the stable convergence of large-stepsize logistic regression might be an artifact due to the functional form of the logistic loss---%
eventually \citet{Tyurin2024} argued that large-stepsize logistic regression behaves like the classical perceptron.
To this end, we show in \cref{theorem:gd} that arbitrary-stepsize GD can converge under a wide range of losses even without the self-bounding property \eqref{equation:self_bounding_property},
and moreover, occasionally yielding a better rate than the classical stable convergence rate.
We discuss it more in \cref{section:discussion}.
Note that one work attempts to extend the separable logistic regression setup to the non-separable one~\cite{Meng2024}; yet, we still do not have satisfactory results beyond the one-dimensional case.
Due to its intricateness, we follow the separable case.

Lastly, our work benefits the study of regret bounds of surrogate losses~\cite{Bartlett2006,Agarwal2014JMLR,Frongillo2021NeurIPS,Bao2023COLT,Mao2023ICML,Bao2024}.
A surrogate regret bound connect a surrogate loss to a downstream task loss, while the optimization error of the surrogate loss is usually ignored.
Our GD convergence analysis can be integrated to surrogate regret bounds when discussing a downstream task performance.

\subsection{Notation}
\label{section:notation}
Let $\Rbb_{\ge0}$ be the set of nonnegative reals.
Let $[n]\defeq\set{1,\dots,n}$ for $n\in\Nbb$.
Let $\onebf$ be the all-ones vector and $\ebf_i\in\Rbb^d$ be the $i$-th standard basis vector, i.e., all zeros except for the $i$-th entry being one.
For $\Scal\subseteq\Rbb^d$, $\interior(\Scal)$ denotes its (relative) interior, and $I_{\Scal}\colon\Rbb\to\set{0,\infty}$ its indicator function, which takes zero if $\mubf\in\Scal$ and $\infty$ otherwise.
For $\Omega\colon\Rbb^d\to\Rbb\cup\set{\infty}$, $\domain(\Omega)\defeq\setcomp{\mubf\in\Rbb^d}{\Omega(\mubf)<\infty}$ denotes its effective domain and $\Omega^*(\thetabf)\defeq\sup\setcomp{\inpr{\thetabf}{\mubf}-\Omega(\mubf)}{\mubf\in\Rbb^d}$ its convex conjugate.
Let $\triangle^d\defeq\setcomp{\mubf\in\Rbb_{\ge0}^d}{\inpr{\onebf}{\mubf}=1}$ be the probability simplex.
We introduce $\Ccal^k(\Ical)$ as the set of $k$-th continuously differentiable functions on the interval $\Ical\subseteq\Rbb$.

Let $\Psi\colon\Rbb^d\to\Rbb\cup\set{\infty}$ be a strictly convex function differentiable throughout $\interior(\domain\Psi)\ne\varnothing$.
We say $\Psi$ is of \emph{Legendre-type} if $\lim_{i\to\infty}\|\nabla\Psi(\xbf_i)\|_2=\infty$ whenever $\xbf_1, \xbf_2, \dots$ is a sequence in $\interior(\domain\Psi)$ converging to a boundary point of $\interior(\domain\Psi)$ (see \cite[Section 26]{Rockafeller1970}).

\section{Preliminary on Fenchel--Young losses}
\label{section:preliminary}
Fenchel--Young losses have been introduced by \citet{Blondel2020JMLR} as a general class of surrogate losses for structured prediction,
which are classification-calibrated~\cite{Wang2024JMLR}.
This can be seen as a Bregman divergence comparing primal and dual points~\cite{Gordon1999COLT,Amari2016}.
Despite that the logistic and hinge losses are widely prevailing in practice, we can improve the performance of some prediction tasks by changing specific Fenchel--Young losses, as reported by \citet{Roulet2025ICML}.
We choose Fenchel--Young losses because a vast majority of convex, Lipschitz, and classification-calibrated losses are included in this class---otherwise, GD convergence is hardly obtained beyond the edge of stability.
Moreover, the separation margin property, one of the key features of Fenchel--Young losses, controls GD behaviors significantly.
\begin{definition}
  \label{definition:fy_loss}
  Let $\Omega\colon\Rbb^K\to\Rbb\cup\set{\infty}$ be a potential function.
  The \emph{Fenchel--Young loss} $\ell_\Omega\colon\domain(\Omega^*)\times\domain(\Omega)\to\Rbb_{\ge0}$ generated by~$\Omega$ is defined as
  \[
    \ell_\Omega(\thetabf;\mubf)\defeq\Omega^*(\thetabf)+\Omega(\mubf)-\inpr{\thetabf}{\mubf}.
  \]
\end{definition}
In multiclass classification, $\ell_\Omega(\thetabf;\mubf)$ measures the proximity between a score $\thetabf$ and a target label $\mubf=\ebf_i$ (for a class $i\in[K]$).
By definition, $\ell_\Omega(\cdot,\mubf)$ is convex for any $\mubf\in\domain(\Omega)$.
Moreover, $\ell_\Omega(\thetabf;\mubf)=0$ holds if and only if $\mubf\in\partial\Omega^*(\thetabf)$ due to the equality condition of the Fenchel--Young inequality.

We follow \cite[Section 4.4]{Blondel2020JMLR} to consider binary ($K=2$) loss functions.
The following set of assumptions is imposed on a potential function $\Omega$.
The asymmetric generalization is possible, but we choose to keep the analysis simpler so that we can focus more on the essence of GD convergence.
\begin{assumption}
  \label{assumption:regularizer}
  For a potential function $\Omega$, assume $\domain(\Omega)\subseteq\triangle^K$ and that $\Omega$ satisfies the zero-entropy condition $\Omega(\mubf)=0$ for $\mubf\in\set{\ebf_i}_{i\in[K]}$;
  convexity $\Omega((1-\alpha)\mubf+\alpha\mubf')\le(1-\alpha)\Omega(\mubf)+\alpha\Omega(\mubf')$ for $\mubf\ne\mubf'$ and $\alpha\in(0,1)$;
  symmetry $\Omega(\mubf)=\Omega(\Pbf\mubf)$ for any $K\times K$ permutation $\Pbf$.
\end{assumption}
Let us restrict ourselves to $K=2$ (binary classification) and write $\phi(\mu)\defeq\Omega([\mu,1-\mu]^\top)$.
If we choose $\thetabf=[s,-s]^\top\in\Rbb^2$ as a score vector, the Fenchel--Young loss can be written as
\[
  \ell_\Omega(\thetabf;\ebf_i)=\begin{cases}
    \phi^*(-s) & \text{if $i=1$,} \\
    \phi^*(s) & \text{if $i=2$,}
  \end{cases}
\]
and $\domain(\phi^*)=\Rbb$.
Hence, the Fenchel--Young loss is simplified as $\phi^*(-ys)$ if we relabel two classes $i=1$ and $i=2$ with $y=1$ and $y=-1$, respectively.
Thus, we suppose the form of a symmetric margin-based loss function~$\ell(z)\defeq\phi^*(-z)$.
Therein, a Fenchel--Young loss~$\ell$ extends a proper canonical composite loss~\cite{Reid2010JMLR} over the entire prediction space~$z\in\Rbb$, as discussed in \cite{Bao2021AISTATS}.

\paragraph{Separation margin.}
For specific potential functions, Fenchel--Young losses entail \emph{separation margin}~\cite[Section 5]{Blondel2020JMLR}, which is a generalized notion of classical margin in support vector machines.
\begin{definition}
  \label{definition:separation_margin}
  For a loss $\ell\colon\Rbb\to\Rbb_{\ge0}$, we say $\ell$ has the \emph{separation margin property} if there exists $m>0$ such that any prediction $z\ge m$ incurs $\ell(z)=0$.
  The smallest $m$ is the \emph{separation margin} of $\ell$.
\end{definition}
Hence, $\ell(z)=0$ indicates the prediction~$z\in\Rbb$ not only correctly classifies a given point but also has safe margin~$m$ away from the classification boundary~$z=0$.
It is shown that the existence of the separation margin property can be tested through the subgradient $\partial\phi$~\cite[Proposition 6]{Blondel2020JMLR}.
\begin{proposition}[{\cite{Blondel2020JMLR}}]
  \label{proposition:separation_margin}
  A Fenchel--Young loss $\ell(z)=\phi^*(-z)$ satisfying \cref{assumption:regularizer} has separation margin if and only if $\partial\phi(\mu)\ne\varnothing$ for any $\mu\in[0,1]$.
  When $\phi\in\Ccal^1((0,1))$ has separation margin $m$,
  \[
    m=-\lim_{\mu\downarrow0}\phi'(\mu).
  \]
\end{proposition}
For a differentiable $\phi$, the nonempty-subgradient condition requires that the derivative $\phi'(\mu)$ does not explode at the boundary points of the domain $\mu\in\domain(\phi)=\set{0,1}$.
In this case, $\phi$ is \emph{not} of Legendre-type~\cite{Rockafeller1970}.
As we will see later, the convergence behavior of GD hinges on the separation margin property of a loss function.
More detailed analysis of the separation margin property for binary classification can be found in \cite{Bao2021AISTATS}.
In \cref{appendix:separation_margin}, we show that a loss satisfying the self-bounding inequality~\eqref{equation:self_bounding_property} does not have separation margin (but not the other way around).

\paragraph{Examples.}
With the Shannon negentropy $\phi(\mu)=\mu\ln\mu+(1-\mu)\ln(1-\mu)$, we recover the logistic loss $\phi^*(-z)=\ln(1+\exp(-z))$.
With the negative of the Gini index $\phi(\mu)=\mu^2-\mu$, we can generate the modified Huber loss $\phi^*(-z)=\max\set{0,1-z}^2/4$ if $z\ge-1$ and $\phi^*(-z)=-z$ otherwise~\cite{Zhang2004ICML}, which is the binarized sparsemax loss~\cite{Martins2016ICML}.
If we choose $\phi(\mu)=\max\set{\mu,1-\mu}$, we recover the hinge loss $\phi^*(-z)=\max\set{0,1-z}$.
We discuss more examples in \cref{section:examples}.

\section{Convergence of large stepsize GD under Fenchel--Young losses}
\label{section:main}
We consistently assume that the dataset is bounded and linearly separable.
\begin{assumption}
  \label{assumption:data}
  Assume the training data $(\xbf_i,y_i)_{i\in[n]}$ satisfies
  \begin{itemize}
    \item for every $i\in[n]$, $\|\xbf_i\|\le1$ and $y_i\in\set{\pm1}$;
    \item there is $\gamma>0$ and a unit vector $\wbf_*$ such that $\inpr{\wbf_*}{\zbf_i}\ge\gamma$ for every $i\in[n]$, where $\zbf_i\defeq y_i\xbf_i$.
  \end{itemize}
\end{assumption}
Instead of logistic regression, we choose a Fenchel--Young loss $\ell(z)=\phi^*(-z)$ associated with a binary potential function $\phi$, and minimize the following risk by \eqref{equation:gd} with fixed stepsize $\eta>0$ to learn a linear classifier $\wbf$:%
\footnote{
  To generalize the linear model, one straightforward way is to focus on deep homogeneous networks~\citep{Tsilivis2025ICLR}.
  We stay on the linear model for now because the straightforward extension to deep homogeneous networks may not significantly change the problem structure.
}
\begin{equation}
  \label{equation:risk}
  L(\wbf)\defeq\frac1n\sum_{i\in[n]}\ell(\inpr{\wbf}{y_i\xbf_i})=\frac1n\sum_{i\in[n]}\ell(\inpr{\wbf}{\zbf_i}).
\end{equation}
We impose the following assumptions on our loss function.
\begin{assumption}
  \label{assumption:loss}
  Consider a loss $\ell\colon\Rbb\to\Rbb_{\ge0}$.
  \begin{assumpenum}
    \item \label{assumption:fy_loss} \textbf{Fenchel--Young loss.} Assume that $\ell(z)$ is a Fenchel--Young loss $\phi^*(-z)$ generated by a potential $\phi:\Rbb\to\Rbb\cup\set{\infty}$
    such that $\phi\in\Ccal^2((0,1))$ satisfies \cref{assumption:regularizer}, $\phi$ is strictly convex, and
    $\phi''>0$ on the interval $(0,1)$.

    \item \label{assumption:regular_loss} \textbf{Regularity.} Assume that $\rho(\lambda)\defeq\min_{z\in\Rbb}\lambda\ell(z)+z^2$ (for $\lambda\ge1$) is well-defined.

    \item \label{assumption:lipschitz_loss} \textbf{Lipschitz continuity.} For $g(\cdot)\defeq|\ell'(\cdot)|$, assume $g(\cdot)\le C_g$ for some $C_g>0$.
  \end{assumpenum}
\end{assumption}
We will later see that $\rho$ characterizes the growth rate of the parameter norm $\|\wbf_t\|$ during GD in~\eqref{equation:norm_ub}.
This notion is inherited from \cite{Wu2024COLT}.
Now, we are ready to state our main result, the GD rate to attain arbitrarily small loss for linearly separable data under Fenchel--Young losses.
Remarkably, we show convergence without the self-bounding property of a loss function, unlike \cite{Wu2024COLT}.
\begin{theorem}[Main result]
  \label{theorem:gd}
  Suppose \cref{assumption:data}
  and consider \eqref{equation:gd} with stepsize $\eta>0$ and $\wbf_0=\zerobf$ under a Fenchel--Young loss $\ell$ satisfying \cref{assumption:loss}.
  For any $\bar\epsilon\in(0,1)$, let
  \begin{equation}
    \label{equation:exponent}
    \alpha\defeq\sup_{\mu\in(0,\bar\epsilon]}\frac{\phi'(\mu)}{\mu\phi''(\mu)}\left[1-\frac{\phi(\mu)}{\mu\phi'(\mu)}\right]
    \quad \text{and} \quad
    C_\phi\defeq\frac{\bar\mu}{[\bar\mu\phi'(\bar\mu)-\phi(\bar\mu)]^\alpha},
  \end{equation}
  where $C_\phi>0$ depends on $\phi$ and $\bar\epsilon$ solely and $\bar\mu\defeq \min\set{g(\ell^{-1}(\bar\epsilon)),1}$.
  If $\alpha,C_\phi\in(0,\infty)$ and
  \[
    \text{for $\epsilon\in(0,\bar\epsilon)$,} \quad
    T > \frac{n}{C_\phi\gamma^2}\left(\frac{4\sqrt{\rho(\gamma^2\eta T)}}{\eta}+C_g\right)\epsilon^{-\alpha}
  \]
  holds, then we have $\min_{t\in[T]}L(\wbf_t)\le\epsilon$.
\end{theorem}
This convergence guarantee even applies to non-smooth Fenchel--Young losses
as long as \cref{assumption:loss} is satisfied---note that $\phi$ must be strongly convex to ensure the smoothness of the associated Fenchel--Young loss~\cite[Proposition~2.4]{Blondel2020JMLR}.
As seen later in \cref{section:examples}, $\alpha$ and $C_\phi$ neither diverge nor degenerate for arbitrarily small $\bar\epsilon$ under many examples of $\phi$.
When $\ell$ has separation margin, \cref{lemma:rho_separation_margin} in \cref{appendix:lemma} provides a finite upper bound on~$\rho$, yielding the following simpler form.
\begin{corollary}
  \label{corollary:gd_separation_margin}
  Under the same setup with \cref{theorem:gd}, we additionally assume that $\ell$ has separation margin $m>0$.
  For any $\bar\epsilon\in(0,1)$, if $(\alpha,C_\phi)$ defined in \eqref{equation:exponent} satisfies $\alpha,C_\phi\in(0,\infty)$ and
  \[
    \text{for $\epsilon\in(0,\bar\epsilon)$,} \quad
    T > \frac{n}{C_\phi\gamma^2}\left(\frac{4m}{\eta}+C_g\right)\epsilon^{-\alpha}
  \]
  holds, then we have $\min_{t\in[T]}L(\wbf_t)\le\epsilon$.
\end{corollary}
A loss function without separation margin does not have finite~$\rho$, which typically yields slower convergence as we see in \cref{section:examples}.
Therefore, \eqref{equation:gd} operated on many common Fenchel--Young losses converges under the separability, regardless of the choice of $\eta$.
Note that the classical GD convergence analysis under convex smooth functions provides $T=\Omega(\epsilon^{-1})$.
As we see later in \cref{section:examples}, some loss functions entail better rates with $\alpha<1$, summarized in \cref{table:loss}.

\subsection{Proof outline}
\label{section:proof_sketch}
The proof of \cref{theorem:gd} essentially relies on the perceptron convergence analysis~\cite{Novikoff1962} and the asymptotical order evaluation of rate functions~\cite{Bao2023COLT}.
We sketch the proof in this section to highlight the structure of the GD convergence in our setup and complete the proof in \cref{proof:gd}.

When we show the convergence of perceptron, we leverage an inequality of the following type:
\begin{equation}
  \label{equation:perceptron_argument}
  \underbrace{C_\text{L}t \le \inpr{\wbf_t}{\wbf_*}}_{(\clubsuit)} \le \underbrace{\|\wbf_t\| \le C_\text{U}(t)}_{(\diamondsuit)}
  \quad \text{for $t\ge1$},
\end{equation}
where $C_\text{L}>0$ is a non-degenerate constant independent of $t$.
The inequality $(\clubsuit)$ holds only while perceptron misclassifies some examples.
Thus, perceptron correctly classifies all examples after at most $T$ iterations such that $C_\text{L}T>C_\text{U}(T)$.
Such $T$ exists as long as $C_\text{U}(t)$ is sublinear in~$T$.

When it comes to our setup, an inequality $(\clubsuit)$ is obtained by recursively expanding the update \eqref{equation:gd}%
\footnote{
  This expansion relies on $\inpr{\wbf_*}{\zbf_{i_{t-1}}}\ge \gamma$, namely, the linear separability in \cref{assumption:data}.
  To lift the linear separability assumption, we may require to introduce additional distributional assumptions here.
}
\[
  \inpr{\wbf_t}{\wbf_*}\ge\inpr{\wbf_{t-1}}{\wbf_*}+\frac{\gamma\eta}{n}g(\inpr{\wbf_{t-1}}{\zbf_{i_{t-1}}})
  \ge \dots \ge \frac{\gamma\eta}{n}\sum_{k=0}^{t-1}g(\inpr{\wbf_k}{\zbf_{i_k}}) \eqdef \gamma\eta\tilde G(\wbf_k),
\]
where $\zbf_{i_k}$ is a misclassified example by $\wbf_k$.
Perceptron enjoys an inequality of $(\clubsuit)$-type immediately because it optimizes the loss function $\ell_{\text{per}}(z)=\max\set{-z,0}$, which yields $g(z)=1$ if $z<0$ (i.e., if misclassified).
When considering a Fenchel--Young loss satisfying \cref{assumption:loss}, we do not have a non-degenerate lower bound for $g(z)$ because we can make $g(z)$ arbitrarily close to zero.
Instead, we lower-bound $g(z)$ by a (non-degenerate) error tolerance $\epsilon_1>0$, $g(z)\ge\epsilon_1$, before we attain the $\epsilon$-optimal loss.
\cref{lemma:w_lb} and (a part of) \cref{lemma:order_evaluation_gd} in \cref{proof:gd} are relevant to $(\clubsuit)$.
Note that the perceptron argument is used in \cite{Wu2024COLT} but in a different way: they control $L(\wbf_k)$ through the upper bound on~$\tilde G(\wbf_k)$ (see \cref{lemma:phase_transition_exponential_tail}).
This is applicable only to self-bounding losses.

To obtain an inequality of $(\diamondsuit)$-type, by following the standard perceptron analysis,
we directly expand the update \eqref{equation:gd} $\|\wbf_t\|^2=\|\wbf_{t-1}-\eta\nabla L(\wbf_{t-1})\|^2$ recursively,
and upper-bound it by noting that $\ell_{\text{per}}$ has separation margin, leading to $C_\text{U}(t)=\Ocal(\sqrt{t})$.
Though this is possible for a Fenchel--Young loss with separation margin, we can improve this bound by borrowing the \emph{split optimization technique}, introduced by~\cite{Wu2024COLT}.
Eventually, we can upper-bound $\|\wbf_t\|$ as follows:
\begin{equation}
  \label{equation:norm_ub}
  \|\wbf_t\| \le \frac{4\sqrt{\rho(\gamma^2\eta t)} + \eta C_g}{\gamma}.
\end{equation}
In particular, we have $\rho(\lambda)=\Ocal(1)$ when a loss has separation margin (see \cref{lemma:rho_separation_margin}), and therein $C_\text{U}(t)=\Ocal(1)$.
This is where separation margin plays a crucial role.
We recap the split optimization technique in \cref{lemma:split_optimization}, based on which \cref{lemma:risk_eos} in \cref{proof:gd_self_bounding} shows this inequality of $(\diamondsuit)$-type.

The remaining piece is to assess the order of the convergence rate.
After solving the inequality~$(\clubsuit,\diamondsuit)$ with $t=T$ being the stopping time, we have $T$ as a function of the error tolerance $\epsilon$, $T=f(\epsilon)$, where $f$ is a nondecreasing rate function depending on $\phi$.
To characterize the asymptotic order at vanishing $\epsilon$, we attempt to evaluate in the form $f(\epsilon)\simeq\epsilon^{\alpha_0}$ for an order parameter $\alpha_0>0$, which can be estimated by
\[
  \frac{\epsilon f'(\epsilon)}{f(\epsilon)} \overset{\epsilon\downarrow0}{\longrightarrow} \alpha_0,
  \quad \text{if the limit exists.}
\]
Thus, the order parameter $\alpha_0$ is solely determined by the functional form of potential function $\phi$.
This technique has been initially developed in functional analysis to estimate moduli of Banach and Orlicz spaces~\cite{Simonenko1964,Hudzik1991,Borwein2009}, and recently introduced in convex analysis to approximate a convex function by power functions~\cite{Ishige2022} and estimate moduli of convexity~\cite{Bao2023COLT,Bao2024}.
The general statement of the order evaluation is given in \cref{lemma:order_evaluation} and instantiated for GD convergence in \cref{lemma:order_evaluation_gd} in \cref{proof:gd}.

\section{Examples of loss functions}
\label{section:examples}

Now, we instantiate \cref{theorem:gd} for several examples of Fenchel--Young losses to discuss the convergence rate.
Instead of specifying a loss function $\ell(z)=\phi^*(-z)$, we directly specify its potential function $\phi$ subsequently.
For each $\phi$, we compute $(\alpha,C_\phi)$ in \eqref{equation:exponent} to investigate the convergence rate given by \cref{theorem:gd}, by taking $\bar\epsilon$ (and thus $\bar\mu$) vanishingly small.
In addition, we can compute separation margin $m$ by \cref{proposition:separation_margin} if exists; otherwise, we need to compute $\rho$ for a loss (see \cref{lemma:rho_estimate}).
\cref{table:loss} summarizes different loss functions and their GD convergence rates.
All the detailed calculations are deferred to \cref{appendix:example},
where we have an additional example of $\phi$ (pseudo-spherical entropy) with non-converging $\alpha$.

\begin{table}
  \centering
  \caption{
    Comparison of Fenchel--Young losses generated by different potential function $\phi$.
    Here, $m=\infty$ and $\beta=\infty$ indicate the lack of separation margin and smoothness, respectively.
    Since we do not have closed-form $\beta$ for the R{\'e}nyi entropy with $q\in(1,2)$, we merely show its lower bound.
    The convergence rates ignore the dependency on $\set{m,n,\gamma,\eta}$, and hold for arbitrary stepsize $\eta$ regardless of $\eta<2/\beta$.
  }
  \label{table:loss}
  \begin{tabular}{cc|cccc}
    \toprule
    {Potential $\phi$} & {Parameter $q$} & {Sep. mgn. $m$} & {Smoothness $\beta$} & {Order $\alpha$} & {Conv. rate for $T$} \\
    \midrule
    {Shannon} & {---} & {$\infty$} & {$1/4$} & {$1$} & {$\tilde\Omega(\epsilon^{-1})$} \\ \midrule
    {Semi-circle} & {---} & {$\infty$} & {$1/4$} & {$2$} & {$\Omega(\epsilon^{-4})$} \\ \midrule
    \multirow{3}{*}{Tsallis} & {$(0,1)$} & {$\infty$} & \multirow{2}{*}{\footnotesize$\dfrac{2^{q-3}}{q}$} & \multirow{2}{*}{\footnotesize$\dfrac1q$} & {$\Omega(\epsilon^{-2/q})$} \\ \cmidrule(l){3-3}
                          {} & {$(1,2]$} & \multirow{2}{*}{\footnotesize$\dfrac{q}{q-1}$} & {} & {} & {$\Omega(\epsilon^{-1/q})$} \\ \cmidrule(l){4-5}
                          {} & {$(2,\infty)$} & {} & {$\infty$} & {$1/2$} & {$\Omega(\epsilon^{-1/2})$} \\ \midrule
    \multirow{3}{*}{R{\'e}nyi} & {$(0,1)$} & {$\infty$} & {$1/4q$} & \multirow{2}{*}{\footnotesize$\dfrac{1}{q}$} & {$\Omega(\epsilon^{-2/q})$} \\ \cmidrule(l){3-3}
                            {} & {$(1,2)$} & \multirow{2}{*}{\footnotesize$\dfrac{q}{q-1}$} & {\color{gray} $(\ge1/4q)$} & {} & {$\Omega(\epsilon^{-1/q})$} \\ \cmidrule(l){5-5}
                            {} & {$2$} & {} & {$\infty$} & {$1/3$} & {$\Omega(\epsilon^{-1/3})$} \\
    \bottomrule
  \end{tabular}
\end{table}

\paragraph{Shannon entropy.}
Consider the binary Shannon (neg)entropy $\phi(\mu)=\mu\ln\mu+(1-\mu)\ln(1-\mu)$.
The generated Fenchel--Young loss is the logistic loss $\ell(z)=\ln(1+\exp(-z))$, which enjoys the self-bounding property and hence does not have separation margin (see \cref{appendix:separation_margin}).
The loss parameters are $\alpha=1$ and $C_\phi=1$.
Moreover, we know $C_g=1$ and $\rho(\lambda)\le1+\ln^2(\lambda)$~\cite{Wu2024COLT}.
Plugging this back to \cref{theorem:gd}, we have the $\epsilon$-optimal risk at most after
\[
  T\gtrsim\left[\frac{4\sqrt2(\log_2(\gamma^2\eta)+1)}{\eta}+\frac{1}{\ln2}\right]\frac{n\epsilon^{-1}}{\gamma^2}
  \quad \text{iterations,}
\]
where logarithmic factors in $\epsilon^{-1}$ are ignored.
This indicates the rate $T=\tilde\Omega(\epsilon^{-1})$, recovering the standard GD convergence rate under the stable regime but with arbitrary stepsize $\eta$.
In \cref{section:discussion}, we compare this rate with \cite{Wu2024COLT} in more detail.

\paragraph{Semi-circle entropy.}
Consider $\phi(\mu)=-2\sqrt{\mu(1-\mu)}$.
The generated Fenchel--Young loss (we call the semi-circle loss)
$\ell(z)=(-z+\sqrt{z^2+4})/{2}$
enjoys the self-bounding property and does not have separation margin since $\phi'(\mu)\to-\infty$ as $\mu\downarrow0$ (see \cref{appendix:separation_margin}).
The semi-circle loss is relevant to the exponential/boosting loss $\ell_\text{exp}(z)=\exp(-z)$, which has the semi-circle entropy as the Bayes risk~\cite{Buja2005,Agarwal2014JMLR}.
The loss parameters are $\alpha=2$ and $C_\phi=1$.
Moreover, we have $C_g=1$ and $\rho(\lambda)\le5\lambda/(2\ln\lambda)$.
Plugging this back to \cref{theorem:gd}, we have the $\epsilon$-optimal risk at most after
\[
  T>\underbrace{\frac{40n^6}{\gamma^2\eta\ln(2\gamma^2\eta)}\epsilon^{-4}}_\text{extra price for lacking separation margin} + \frac{2n}{\gamma^2}\epsilon^{-2}
  \quad \text{iterations,}
\]
where the first term $\Omega(\epsilon^{-4})$ is an extra price due to the lack of separation margin of the semi-circle loss.
For arbitrary stepsize $\eta$, the convergence rate is $T=\Omega(\epsilon^{-4})$, and stepsize $\eta$ as large as $\eta=\Omega(\epsilon^{-2})$ improves the rate to be $T=\tilde\Omega(\epsilon^{-2})$ by cancelling the extra term out.

This convergence rate of the semi-circle loss is even worse than the GD convergence rate for general convex smooth functions, $T=\Omega(\epsilon^{-1})$.
This is because the perceptron argument is merely sufficient for GD convergence.
Nonetheless, the perceptron argument more informatively states that we have $\inpr{\wbf_t}{\wbf_*}/\|\wbf_t\|\gtrsim\epsilon^\alpha$ after minimizing the risk at the $\epsilon$-optimal level---%
by combining the inequalities $(\clubsuit,\diamondsuit)$ (in \cref{equation:perceptron_argument}).
This indicates that the loss function with larger $\alpha$ yields slower parameter alignment toward $\wbf_*$.

\paragraph{Tsallis entropy.}
For $q>0$ with $q\ne1$, consider the Tsallis $q$-(neg)entropy
\[
  \phi(\mu)=\frac{\mu^q+(1-\mu)^q-1}{q-1}
\]
generalizing the Shannon entropy for non-extensive systems~\cite{Tsallis1988}.
It recovers the Shannon entropy at the limit $q\to1$.
The generated Fenchel--Young loss is known as the $q$-entmax loss~\cite{Peters2019ACL}.
We divide the case depending on parameter $q$:
\begin{itemize}
  \item \underline{When $0<q<1$:}
  $(\alpha,C_\phi)=(1/q, 1)$, and $\phi^*$ does not have separation margin.

  \item \underline{When $1<q\le2$:}
  $(\alpha,C_\phi)=(1/q, 1)$, and $\phi^*$ has separation margin $m=q/(q-1)$.

  \item \underline{When $2<q$:}
  $(\alpha,C_\phi)=(1/2, \sqrt{2/q})$, and $\phi^*$ has separation margin $m=q/(q-1)$.
\end{itemize}
For all cases, $\alpha$ and $C_\phi$ stay in $(0,\infty)$.
The convergence rate is $T=\Omega(\epsilon^{-2/q})$ for $q\in(0,1)$ (by \cref{corollary:gd_no_separation_margin}); $T=\Omega(\epsilon^{-1/q})$ for $q\in(1,2)$; $T=\Omega(\epsilon^{-1/2})$ for $2\le q$.
This suggests that we have a better convergence rate over the Shannon case when $q>1$ and the best rate is $\Omega(\epsilon^{-1/2})$.

\paragraph{R{\'e}nyi entropy.}
For $q\in(0,2]\setminus\set{1}$, consider the R{\'e}nyi $q$-(neg)entropy
\[
  \phi(\mu)=\frac{1}{q-1}\ln\left[\mu^q+(1-\mu)^q\right]
\]
generalizing the Shannon entropy (with the limit $q\to1$) while preserving additivity for independent events~\cite{Renyi1961}.
The R{\'e}nyi entropy extended beyond $q>2$ becomes nonconvex, which we do not consider.
The R{\'e}nyi $2$-entropy is referred to as the collision entropy~\cite{Bosyk2012}.

We divide the case depending on parameter $q$:
\begin{itemize}
  \item \underline{When $0<q<1$:}
  $(\alpha,C_\phi)=(1/q,1)$, and $\phi^*$ does not have separation margin.

  \item \underline{When $1<q<2$:}
  $(\alpha,C_\phi)=(1/q,1)$, and $\phi^*$ has separation margin $m=q/(q-1)$.

  \item \underline{When $q=2$:}
  $(\alpha,C_\phi)=(1/3,\sqrt[3]{3/8})$, and $\phi^*$ has separation margin $m=2$.
\end{itemize}
For all cases, $\alpha$ and $C_\phi$ stay in $(0,\infty)$.
The convergence rate is $T=\Omega(\epsilon^{-2/q})$ for $q\in(0,1)$ (by \cref{corollary:gd_no_separation_margin});
$T=\Omega(\epsilon^{-1/q})$ for $q\in(1,2)$;
$T=\Omega(\epsilon^{-1/3})$ for $q=2$.
Surprisingly, we have a ``leap'' of the order from $\epsilon^{-1/q}$ to $\epsilon^{-1/3}$ as $q\uparrow2$, and the convergence rate $\Omega(\epsilon^{-1/3})$ is far better than the Shannon and Tsallis cases.
When $q=2$, \cref{corollary:gd_separation_margin} implies that we have the $\epsilon$-optimal risk at most after
\[
  T>\sqrt[3]{8/3}\frac{n}{\gamma^2}\left(\frac{8}{\eta}+1\right)\epsilon^{-1/3} \quad \text{iterations.}
\]

\section{Discussion and open problems}
\label{section:discussion}

\paragraph{Comparison with \citet{Wu2024COLT}.}
The large-stepsize logistic regression has been shown to exhibit the following phase transition~\cite{Wu2024COLT}:
the GD sequence initially stays in the EoS phase such that the risk~$L(\wbf_t)$ fluctuates initially with its average~$t^{-1}\sum_kL(\wbf_k)$ controlled.
Once we experience~$L(\wbf_t)\lesssim\min\set{1/\eta, \ell(0)/n}$, which is possible within~$\Ocal(\eta)$ steps at most,
GD leaves the EoS and the loss converges in the rate~$L(\wbf_t)=\tilde\Ocal(1/(\eta t))$.
The stepsize~$\eta$ trades off the phase transition time for the stable convergence rate.
If we know the maximum number of steps $T$ in advance, the choice $\eta=\Theta(T)$ balances them, achieving the acceleration to $L(\wbf_t)=\tilde\Ocal(1/T^2)$.
We detail them in \cref{proof:gd_self_bounding}.
This is arguably interesting to demonstrate how GD benefits from large stepsize.

Nevertheless, we would like to highlight two caveats.
First, we must undergo $L(\wbf_t)\le\ell(0)/n$ before exiting the EoS phase.
This means that \emph{our linear model has already classified all points correctly during the EoS phase} since any single point $\zbf_i$ incurs loss at most $\ell(\inpr{\wbf_s}{\zbf_i})\le\ell(0) \implies \inpr{\wbf_s}{\zbf_i}\ge0$ (cf. \cref{lemma:correct_classification} in \cref{proof:gd_self_bounding}).%
\footnote{
  \citet{Cai2024NeurIPS} extends \citet{Wu2023NeurIPS} for two-layer near-homogeneous NNs, where it is not explicit that the model correctly classifies all points after the EoS phase.
  Taking a closer look, we can see that their Lemma~A.7 leverages the well-controlled risk, which is an alternative expression to ``$L(\wbf_s)\le\ell(0)/n$'' under their setup.
}
GD keeps improving the logistic loss after the stable phase just because the logistic loss does not enjoy separation margin and never touches strict zero.
We refer interested readers to the relevant discussion in \citet{Tyurin2024}, who argues that the faster convergence in the stable phase is attributed to the choice of the logistic loss.

Second, the EoS termination condition~$L(\wbf_t)\lesssim1/\eta$ suggests that the risk must be once $\Ocal(1/T)$-optimal (with the optimally balancing choice $\eta=\Theta(T)$) before benefitting from the super-fast rate $\tilde\Ocal(1/T^2)$.
Yet, GD under some loss functions including the Tsallis $q$-loss ($q>1$) and the R{\'e}nyi $q$-loss ($q>1$) achieves better risk with the same GD steps.
If our goal is simply to classify all training points, these alternative losses might do better jobs in terms of optimization solely.

\begin{figure}
  \centering
  \includegraphics[width=0.3\textwidth]{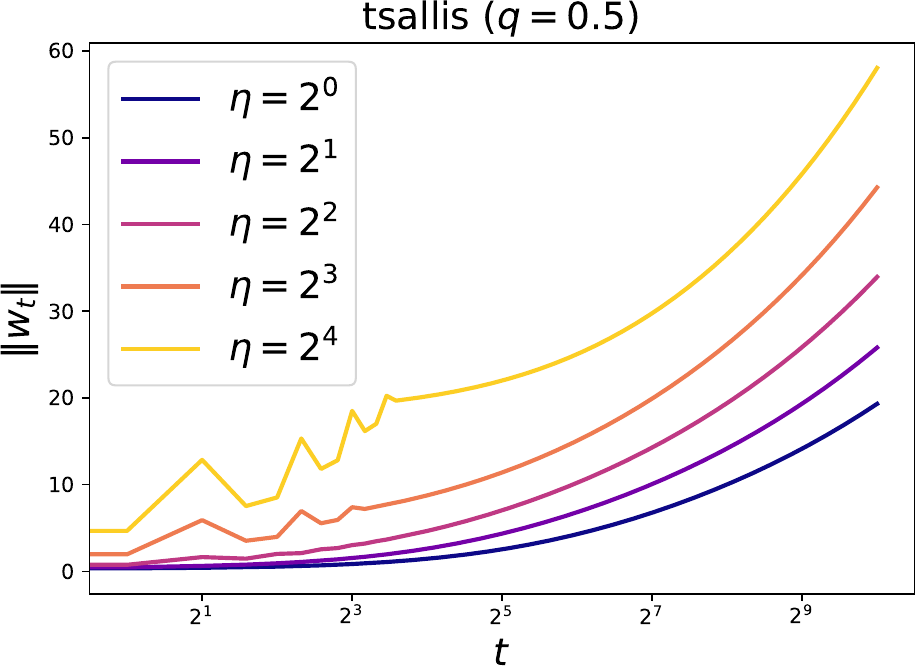} \hspace{5pt}
  \includegraphics[width=0.3\textwidth]{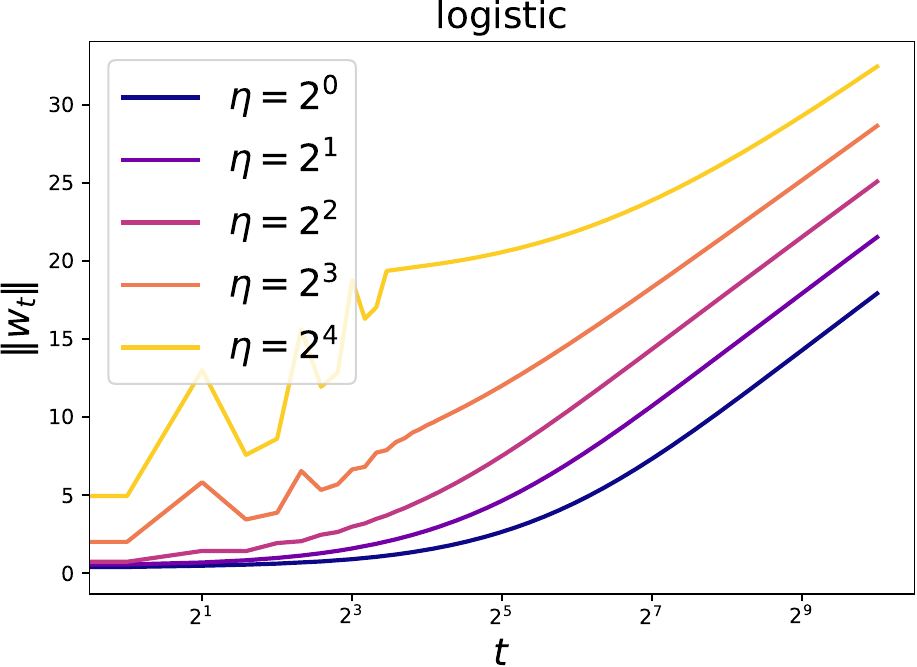} \hspace{5pt}
  \includegraphics[width=0.3\textwidth]{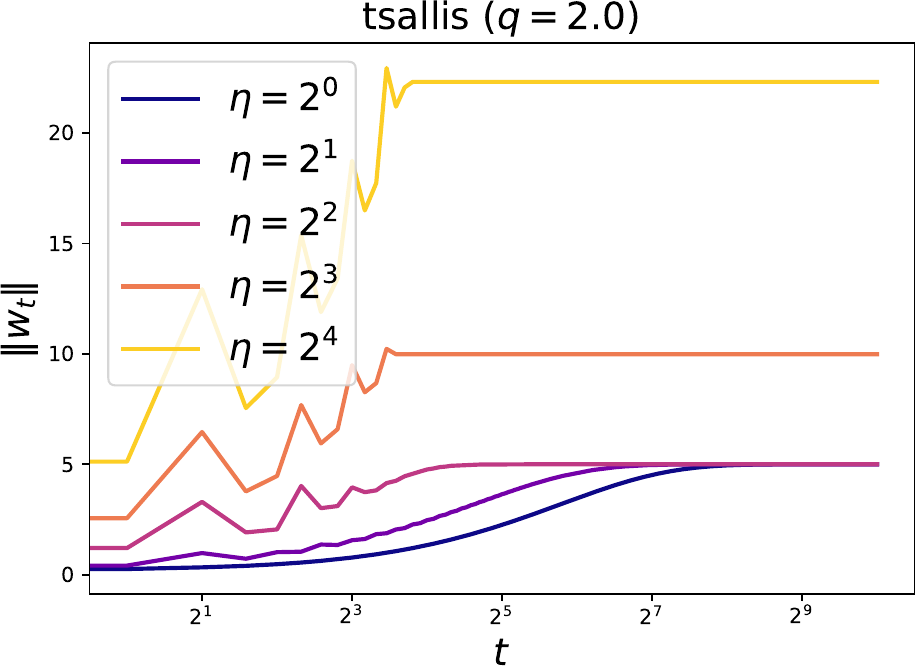}
  \caption{
    Under the same setup as \cref{figure:pilot}, we show $\|\wbf_t\|$ along the number of steps $t$ with different losses.
  }
  \label{figure:norm}
\end{figure}

\paragraph{Self-bounding property and implicit bias.}
Having said that, the phase transition may play an important role in implicit bias.
It was shown under the linearly separable case that logistic regression optimized with GD enlarges the norm $\|\wbf_t\|$ toward the max-margin direction in rate $\Omega(\ln(t))$~\cite{Soudry2018,Wu2023NeurIPS,Cai2024NeurIPS}.
Thus, we may argue that $\wbf_t$ gradually comes to classify all data points correctly during the EoS phase and evolves toward the max-margin direction in the stable phase.

We reported how $\|\wbf_t\|$ evolves under the pilot setup in \cref{figure:norm} with different loss functions.
As seen, the logistic and Tsallis~$0.5$ losses inflate $\|\wbf_t\|$ endlessly, which do not have separation margin.
In stark contrast, the Tsallis $2$-loss prevents $\|\wbf_t\|$ from growing endlessly just because of its separation margin---%
recall the norm upper bounds of the norm $\|\wbf_t\|$ in \eqref{lemma:w_lb} and the growth rate $\rho$ in \cref{lemma:rho_separation_margin}.
This raises two open questions:
(1) Do we have similar implicit bias aligning toward the max-margin direction under a loss function with separation margin?
(2) What are benefits and caveats of endless growing of $\|\wbf_t\|$?
The latter is particularly relevant to the overconfidence issue due to excessively large $\|\wbf_t\|$~\cite{Wei2022ICML} and worse generalization due to prohibitively large within-class variance~\cite{Hou2022NeurIPS}.
\citet{Wu2025} argues that excessively large $\|\wbf_t\|$ leads to an inconsistent estimator.

The study on implicit bias for loss functions with the self-bounding property has been very scarce.
To our knowledge, \cite{Lizama2020} crafted the complete hinge loss,
which behaves like the hinge loss before GD converges to the zero risk yet incurs an extra penalty to artificially align the parameter toward the max-margin direction.
Together with the benefits and caveats of the max-margin implicit bias, we believe this is an interesting open topic.

\paragraph{Dependency on $n$.}
Our main result (\cref{theorem:gd}) provides the rate depending on the factor $n$.
This extra factor with respect to $n$ arises due to the worst-case analysis such that we have \emph{at least} one ``bad'' direction $\zbf_i$ before the convergence, corresponding to the inequality~\eqref{equation:gradient_lower_bound} in the proof of \cref{theorem:gd} (see \cref{proof:gd}).
This worst-case scenario supposes that all data points are nearly equidistant, which is unlikely since most data points tend to cluster in similar directions.
We conjecture that this $n$-dependency is not essential with additional mild data assumptions.

\paragraph{The stochastic case.}
Our result can be extended for the stochastic gradient descent (SGD).
Consider the scenario where we sample one fresh data point at each $t$ and update the linear parameter with the loss function computed on this sample.
Under the similar setup to \cref{theorem:gd}, the population risk is $\epsilon$-optimal with high probability after $T=\Omega(\epsilon^{-(\alpha+2)})$.
The formal statement and proof are shown in \cref{section:sgd}.
While this rate apparently looks significantly slower than the GD rate $T=\Omega(\epsilon^{-\alpha})$,
the extra iterations $\epsilon^{-2}$ is necessary for collecting sufficient samples to estimate the population risk.
By noting that the GD/SGD updates consume $n$/one samples, the GD/SGD rates are comparable in terms of the number of consumed samples.

\paragraph{Finite-time convergence.}
Last but not least, we may potentially have another benefit of loss functions with separation margin.
Take a look at \cref{figure:pilot} again.
Loss functions without separation margin, such as the Tsallis $1.5$- and $2$-losses, converge to exact zero within finite time when sufficiently large stepsize is used.
Such finite-time convergence under the linearly separable case can be shown without significant challenges if we use perceptron, or even the hinge loss, while becoming highly non-trivial in the case of twice-differentiable loss functions.
This is because the perceptron argument requires a non-degenerate lower bound on $\inpr{\wbf_t}{\wbf_*}$ (see \eqref{equation:perceptron_argument}), which is not straightforward therein as the loss gradient can be arbitrarily small positive (due to the twice differentiability of the loss).
We conjecture that an additional data assumption is necessary because the loss gradient could be adversarially vanishing against GD convergence, and leave this as future work.

% Acknowledgments---Will not appear in anonymized version
\section*{Acknowledgments}
HB is supported by JST PRESTO (Grant No. JPMJPR24K6).
SS was supported by JST ERATO (Grant No. JPMJER1903).
YT is supported by JSPS KAKENHI (Grant No. 23KJ1336).

\bibliography{reference}
\bibliographystyle{plainnat}

%%%%%%%%%%%%%%%%%%%%%%%%%%%%%%%%%%%%%%%%%%%%%%%%%%%%%%%%%%%%
%%%%%%%%%%%%%%%%%%%%%%%%%%%%%%%%%%%%%%%%%%%%%%%%%%%%%%%%%%%%
%%%%%%%%%%%%%%%%%%%%%%%%%%%%%%%%%%%%%%%%%%%%%%%%%%%%%%%%%%%%

\newpage
\appendix

\section{Technical lemmas}
\label{appendix:lemma}

We introduce the gradient potential for a loss function in consideration as follows:
\begin{equation}
  \label{equation:gradient_potential}
  G(\wbf)\defeq\frac1n\sum_{i=1}^ng(\inpr{\wbf}{\zbf_i}).
\end{equation}

\begin{lemma}
  \label{lemma:loss_bound}
  Consider a loss $\ell$ satisfying \cref{assumption:regular_loss}.
  Then, we have
  \[
    \ell\left(\sqrt{\rho(\lambda)}\right) \le \frac{\rho(\lambda)}{\lambda}.
  \]
\end{lemma}

\begin{proof}
  See \cite[Lemma 20]{Wu2024COLT}.
\end{proof}

\begin{lemma}
  \label{lemma:loss_monotone}
  Consider a Fenchel--Young loss $\ell(z)=\phi^*(-z)$ satisfying \cref{assumption:fy_loss}.
  Then, $\ell$ and $g$ are nonincreasing.
  Moreover, $\ell$ is strictly decreasing on $(-\infty,m)(\subseteq\Rbb)$ if $\ell$ has separation margin $m>0$;
  otherwise, $\ell$ is strictly decreasing on $\Rbb$.
\end{lemma}

\begin{proof}
  We have by Danskin's theorem~\cite{Danskin1966} $(\phi^*)'=(\phi')^{-1}$, and then
  \[
    \ell'(z) = -(\phi^*)'(-z) = -\underbrace{(\phi')^{-1}(-z)}_{\in\domain(\phi)\subseteq[0,1]} \le 0,
  \]
  which implies that $\ell$ is nonincreasing.
  Since $\ell$ is convex and nonincreasing we have that $g(\cdot)=|\ell'(\cdot)|=-\ell'(\cdot)$ is nonincreasing.

  For the latter part, if nonincreasing $\ell$ has separation margin, $\ell(z)=0$ if and only if $z\ge m$.
  Then, we have $\ell\equiv0$ on the interval $[m,\infty)\subseteq\Rbb$ and $\ell>0$ on the interval $(-\infty,m)\subseteq\Rbb$.
  On the latter interval, $\ell$ must be strictly decreasing because of its convexity.
  We can prove similarly for $\ell$ lacking separation margin.
\end{proof}

\begin{lemma}
  \label{lemma:rho_separation_margin}
  Consider a loss $\ell$ satisfying \cref{assumption:regular_loss}.
  Suppose that $\ell$ has separation margin $m>0$.
  Then, we have $\rho(\lambda) \le m^2$ for any $\lambda\ge1$.
\end{lemma}

\begin{proof}
  When $\ell$ has separation margin $m$ (see \cref{definition:separation_margin}), we have
  \[
    \lambda\ell(z)+z^2 = z^2 \quad \text{for $z\ge m$.}
  \]
  By the definition of $\rho$, we have
  \[
    \begin{aligned}
      \rho(\lambda)
      = \min_{z\in\Rbb}\lambda\ell(z)+z^2
      \le \min_{z\ge m}z^2
      = m^2.
    \end{aligned}
  \]
\end{proof}

\begin{lemma}[{Split optimization~\cite{Wu2024COLT}}]
  \label{lemma:split_optimization}
  Suppose \cref{assumption:data}
  and consider a convex and nonincreasing loss $\ell$ satisfying \cref{assumption:lipschitz_loss} and let $\ubf\defeq\ubf_1+\ubf_2$ such that
  \[
    \ubf_1=\theta\wbf_*, \qquad \ubf_2=\frac{\eta C_g}{2\gamma}\wbf_*.
  \]
  For every $t\ge1$, we have
  \[
    \frac{\|\wbf_t-\ubf\|^2}{2\eta t} + \frac1t\sum_{k=0}^{t-1}L(\wbf_k)
    \le \ell(\gamma\theta) + \frac{1}{2\eta t}\left(\theta+\frac{\eta C_g}{2\gamma}\right)^2.
  \]
\end{lemma}

\begin{proof}
  For $k<t$, we have
  \[
    \begin{aligned}
      \|\wbf_{k-1}-\ubf\|^2
      &= \|\wbf_k-\ubf\|^2 + 2\eta\inpr{\nabla L(\wbf_k)}{\ubf-\wbf_k} + \eta^2\|\nabla L(\wbf_k)\|^2 \\
      &= \|\wbf_k-\ubf\|^2 + 2\eta\inpr{\nabla L(\wbf_k)}{\ubf_1-\wbf_k} + \eta(2\inpr{\nabla L(\wbf_k)}{\ubf_2} + \eta\|\nabla L(\wbf_k)\|^2).
    \end{aligned}
  \]
  For the second term, we have
  \begin{equation}
    \label{equation:proof:supp:2}
    \begin{aligned}
      \inpr{\nabla L(\wbf_k)}{\ubf_1-\wbf_k}
      &= \frac1n\sum_{i=1}^n\ell'(\inpr{\wbf_k}{\zbf_i})\inpr{\zbf_i}{\ubf_1-\wbf_k} \\
      &= \frac1n\sum_{i=1}^n\ell'(\inpr{\wbf_k}{\zbf_i})(\inpr{\ubf_1}{\zbf_i} - \inpr{\wbf_k}{\zbf_i}) \\
      &\le \frac1n\sum_{i=1}^n\left[\ell(\inpr{\ubf_1}{\zbf_i}) - \ell(\inpr{\wbf_k}{\zbf_i})\right]
        && \text{($\ell$: convex)} \\
      &\le \ell(\gamma\theta) - L(\wbf_k).
        && \text{($\ell$: nonincreasing)}
    \end{aligned}
  \end{equation}
  For the third term, we have
  \[
    \begin{aligned}
      &2\inpr{\nabla L(\wbf_k)}{\ubf_2} + \eta\|\nabla L(\wbf_k)\|^2 \\
      &= \frac2n\sum_{i=1}^n\ell'(\inpr{\wbf_k}{\zbf_i})\inpr{\zbf_i}{\ubf_2} + \eta\left\|\frac1n\sum_{i=1}^n\ell'(\inpr{\wbf_k}{\zbf_i})\zbf_i\right\|^2 \\
      &\le \frac2n\sum_{i=1}^n\ell'(\inpr{\wbf_k}{\zbf_i})\inpr{\zbf_i}{\ubf_2} + \eta\left(\frac1n\sum_{i=1}^n\ell'(\inpr{\wbf_k}{\zbf_i})\right)^2
        && \text{($\|\zbf_i\|\le1$)} \\
      &\le \frac{2\|\ubf_2\|}{n}\sum_{i=1}^n\ell'(\inpr{\wbf_k}{\zbf_i})\inpr{\zbf_i}{\wbf_*} + \eta C_g\cdot G(\wbf_k)
        && \text{(\cref{assumption:lipschitz_loss} and $G(\wbf_k)\ge0$)} \\
      &\le -2\gamma\|\ubf_2\|G(\wbf_k) + \eta C_g\cdot G(\wbf_k)
        && \text{(\cref{assumption:data} and $G(\wbf_k)\ge0$)} \\
      &= 0,
    \end{aligned}
  \]
  where the last equality is by the choice of $\ubf_2$ and $G(\wbf)$ is defined in \eqref{equation:gradient_potential}.

  By combining them altogether, we have for $k<t$,
  \[
    \|\wbf_{k+1}-\ubf\|^2 \le \|\wbf_k-\ubf\|^2 + 2\eta\left[\ell(\gamma\theta) - L(\wbf_k)\right].
  \]
  Telescoping the sum from $0$ to $t-1$ and rearranging, we get
  \[
    \frac{\|\wbf_t-\ubf\|^2}{2\eta t} + \frac1t\sum_{k=0}^{t-1}L(\wbf_k)
    \le \ell(\gamma\theta) + \frac{\|\wbf_0-\ubf\|^2}{2\eta t},
  \]
  which completes the proof.
\end{proof}

\section{Proof of Theorem~\ref{theorem:gd}}
\label{proof:gd}

\begin{lemma}
  \label{lemma:w_lb}
  Suppose \cref{assumption:data}
  and consider \eqref{equation:gd} with any stepsize $\eta>0$ under a Fenchel--Young loss $\ell$ that satisfies \cref{assumption:fy_loss}.
  For $t\ge1$, assume $G(\wbf_k)\ge G_{\min} > 0$ for all $k\in[0,t-1]$, where $G(\wbf)$ is defined in \eqref{equation:gradient_potential}.
  Then, we have
  \[
    \gamma\eta G_{\min}t\le\inpr{\wbf_t}{\wbf_*} - \inpr{\wbf_0}{\wbf_*}.
  \]
\end{lemma}

\begin{proof}
  By the perceptron argument~\cite{Novikoff1962}, we have
  \[
    \begin{aligned}
      \inpr{\wbf_{k+1}}{\wbf_*}
      &= \inpr{\wbf_k}{\wbf_*} - \eta\inpr{\nabla L(\wbf_k)}{\wbf_*} \\
      &= \inpr{\wbf_k}{\wbf_*} - \frac\eta n\inpr{\sum_{i=1}^n\ell'(\inpr{\wbf_k}{\zbf_i})\zbf_i}{\wbf_*} \\
      &= \inpr{\wbf_k}{\wbf_*} + \frac\eta n\sum_{i=1}^ng(\inpr{\wbf_k}{\zbf_i})\inpr{\wbf_*}{\zbf_i}
        && \text{(use $g(\cdot)=-\ell'(\cdot)$ by \cref{lemma:loss_monotone})} \\
      &\ge \inpr{\wbf_k}{\wbf_*} + \gamma\eta G(\wbf_k)
        && \text{(note $g(\cdot)\ge0$ and \cref{assumption:data})} \\
      &\ge \inpr{\wbf_k}{\wbf_*} + \gamma\eta G_{\min}.
    \end{aligned}
  \]
  Telescoping the sum, we have the desired inequality.
\end{proof}

\begin{lemma}[Order evaluation]
  \label{lemma:order_evaluation}
  Let $\Ical\subseteq\Rbb_{\ge0}$ be an open interval containing zero as the left end.
  For $f\colon\Rbb_{\ge0}\to\Rbb_{\ge0}\cup\set{\infty}$ that is nondecreasing and differentiable on $\Ical$ and satisfies $f(0)=0$, let
  \begin{equation}
    \label{equation:exponent_def}
    \alpha\defeq\sup_{x\in(0,x_0]}\frac{xf'(x)}{f(x)} \quad \text{for some $x_0\in\Ical$.}
  \end{equation}
  Then, for any $x\in(0,x_0)$, we have
  \[
     f(x)\ge Cx^\alpha, \quad \text{where} \quad C\defeq\frac{f(x_0)}{x_0^\alpha}.
  \]
\end{lemma}

\begin{proof}
  By the definition of $\alpha$, we have
  \[
    \alpha \ge \frac{xf'(x)}{f(x)} \quad \text{for all $x\in(0,x_0]$.}
  \]
  Then, for any $x\in(0,x_0)$, we have
  \[
    \alpha\ln\frac{x_0}{x} = \alpha\int_{x}^{x_0}\frac{\rd{s}}{s}\ge\int_x^{x_0}\frac{f'(s)}{f(s)}\rd{s} = \ln\frac{f(x_0)}{f(x)}.
  \]
  By reorganizing this inequality, we can prove the original argument.
\end{proof}

\begin{lemma}
  \label{lemma:order_evaluation_gd}
  Consider a Fenchel--Young loss $\ell(z)=\phi^*(-z)$ satisfying \cref{assumption:fy_loss}.
  Then, for arbitrary $0<\bar\epsilon<1$ and $\alpha$ defined in \eqref{equation:exponent}, we have
  \[
    g(\ell^{-1}(\epsilon)) \ge C_\phi\epsilon^{\alpha} \quad \text{for $\epsilon\in(0,\bar\epsilon)$}, \quad
    \text{where} \quad C_\phi\defeq \frac{g(\ell^{-1}(\bar\epsilon))}{\bar\epsilon^{\alpha}}.
  \]
\end{lemma}

\begin{proof}
  Choose any $\epsilon_0>0$.
  For $\epsilon\in(0,\epsilon_0)$, we can invert to get $z\equiv\ell^{-1}(\epsilon)$ because $\ell$ is strictly decreasing on $\ell^{-1}((0,1))$ (by \cref{lemma:loss_monotone})
  and hence invertible.
  Note that
  \[
    \begin{cases}
      z\in(\ell^{-1}(\epsilon_0),m) & \text{if $\ell$ has separation margin $m>0$,} \\
      z\in(\ell^{-1}(\epsilon_0),\infty) & \text{otherwise,}
    \end{cases}
  \]
  because $\ell^{-1}(\cdot)$ is nonincreasing.
  We write this range as $\Ical$,
  then $z=\ell^{-1}(\epsilon)\in\Ical$ for $\epsilon\in(0,\epsilon_0)$.

  Let us verify $g(z)=-\ell'(z)$ is differentiable at $z\in\Ical$ first.
  By the definition of Fenchel--Young losses, we have
  \[
    \ell(z)
    =\phi^*(-z)
    =\sup_{x\in(0,1)}\left[x\cdot(-z)-\phi(x)\right]
    =-z\cdot(\phi')^{-1}(-z)-\phi\left((\phi')^{-1}(-z)\right),
  \]
  for $z\in\Ical$,
  where we used the first-order optimality $-z=\phi'(x)$ of the convex conjugate at the last identity.
  Since $\phi$ is twice continuously differentiable and $\phi''>0$ on the interval $(0,1)$ by \cref{assumption:fy_loss},
  we can apply the inverse function theorem to have
  \[
    \ell'(z)
    =-(\phi')^{-1}(-z) - \frac{z}{\phi''\left((\phi')^{-1}(-z)\right)} - \frac{\phi'\left((\phi')^{-1}(-z)\right)}{\phi''\left((\phi')^{-1}(-z)\right)}
    =-(\phi')^{-1}(-z),
  \]
  for $z\in\Ical$.
  Since $\phi$ is twice continuously differentiable, we can apply the inverse function theorem once again to get $\ell''(z)$ for $z\in\Ical$, and hence $g$ is differentiable on $\Ical$.

  In addition, $\ell$ is continuously differentiable with non-degenerate derivative at $z\in\Ical$
  because $\ell$ is strictly decreasing on $\Ical$.
  From these observations, we can see that $g(\ell^{-1}(\cdot)) \eqdef f(\cdot)$ is nondecreasing on $\Ical$ and differentiable on $\Ical$
  (because it is the composition of two nondecreasing and differentiable functions $g$ and $\ell^{-1}$).
  Now we can apply \cref{lemma:order_evaluation} to this $f$.
  Let us compute the exponent $\alpha_\epsilon$ defined in \eqref{equation:exponent_def}.
  By the differentiability and non-degenerate derivative of $\ell$ on $\Ical$, we can apply the inverse function theorem on $\ell$ to have
  \begin{align*}
    \frac{\epsilon f'(\epsilon)}{f(\epsilon)}
    &= \frac{\epsilon}{g(\ell^{-1}(\epsilon))} \cdot g'(\ell^{-1}(\epsilon)) \cdot \frac{1}{\ell'(\ell^{-1}(\epsilon))}
      && \text{(inverse function theorem)} \\
    &= \frac{\ell(z)g'(z)}{g(z)\ell'(z)}
      && \text{($\epsilon\equiv\ell(z)$)} \\
    &= \frac{\ell(z)\ell''(z)}{[\ell'(z)]^2}
      && \text{($g(\cdot)=-\ell'(\cdot)$)} \\
    &= \frac{\phi^*(\bar z)\cdot(\phi^*)''(\bar z)}{[(\phi^*)'(\bar z)]^2}
      && \text{($\bar z\defeq-z$)} \\
    &\overset{\text{(A)}}{=} \frac{[\mu\phi'(\mu)-\phi(\mu)] \cdot \frac{1}{\phi''(\mu)}}{\mu^2}
      && \text{($\mu\equiv(\phi^*)'(\bar z)$)} \\
    &= \frac{\phi'(\mu)}{\mu\phi''(\mu)}\left[1-\frac{\phi(\mu)}{\mu\phi'(\mu)}\right],
  \end{align*}
  where at (A) we introduce $\mu$ as the dual of $\bar z$ by the mirror map $\phi'$ such that
  \[
    \bar z=\phi'(\mu) \quad \text{and} \quad \mu=(\phi^*)'(\bar z),
  \]
  which implies $\phi^*(\bar z)=\mu\phi'(\mu)-\phi(\mu)$ together with the definition of the convex conjugate,
  and
  \[
    [\phi''(\mu)]\cdot[(\phi^*)''(\bar z)]=1
  \]
  with Danskin's theorem~\cite{Danskin1966} and the inverse function theorem.
  Note that this identity is often referred to as Crouzeix's identity~\cite{Crouzeix1977}.
  Here, we have
  \[
    \begin{cases}
      \bar z\in\left(-m,-\ell^{-1}(\epsilon_0)\right) \text{~~and~~} \mu\in\left(0,g(\ell^{-1}(\epsilon_0))\right) & \text{if $\ell$ has separation margin $m>0$,} \\
      \bar z\in\left(-\infty,-\ell^{-1}(\epsilon_0)\right) \text{~~and~~} \mu\in\left(0,g(\ell^{-1}(\epsilon_0))\right) & \text{otherwise,}
    \end{cases}
  \]
  by noting that $g(z)=-\ell'(z)=(\phi^*)'(-z)$ is nonincreasing.
  With this primal-dual relationship, we have
  \[
    \phi^*(\bar z) = \mu\bar z - \phi(\mu) \quad \text{and} \quad
    [\phi''(\mu)] \cdot [(\phi^*)''(\bar z)] = 1
  \]
  by the definition of the convex conjugate and Crouzeix's identity.
  Now we are ready to apply \cref{lemma:order_evaluation}, which yields
  \[
    g(\ell^{-1}(\epsilon)) \ge C_\phi\epsilon^{\alpha} \quad \text{for $\epsilon\in(0,\bar\epsilon)$},
  \]
  where
  \[
    \alpha \defeq \sup_{\mu\in(0,\bar\epsilon]} \frac{\phi'(\mu)}{\mu\phi''(\mu)}\left[1-\frac{\phi(\mu)}{\mu\phi'(\mu)}\right],
    \quad
    C_\phi \defeq \frac{g(\ell^{-1}(\bar\epsilon))}{\bar\epsilon^{\alpha}},
    \quad \text{and} \quad
    \bar\epsilon\defeq g(\ell^{-1}(\epsilon_0)).
  \]
  Since the choice of $\epsilon_0>0$ was arbitrary and $\image(g)=\image((\phi^*)')=\domain(\phi')\subseteq[0,1]$,
  we can choose such $\bar\epsilon\in(0,1)$.
\end{proof}

{\renewcommand{\proofname}{Proof of \cref{theorem:gd}.}
\begin{proof}
  For a fixed $k\in[T-1]$, if we have $L(\wbf_k)>\epsilon$, there exists $i\in[n]$ such that $\ell(\inpr{\wbf_k}{\zbf_i})>\epsilon$.
  Then, we have for this specific $i\in[n]$,
  \[
    \begin{aligned}
      & \inpr{\wbf_k}{\zbf_i} < \ell^{-1}(\epsilon) && \text{($\ell$ is strictly decreasing when $\ell>0$ by \cref{lemma:loss_monotone})} \\
      \implies & g(\inpr{\wbf_k}{\zbf_i}) \ge g(\ell^{-1}(\epsilon)). && \text{($g$ is nonincreasing by \cref{lemma:loss_monotone})}
    \end{aligned}
  \]
  This implies that
  \begin{equation}
    \label{equation:gradient_lower_bound}
    G(\wbf_k)=\frac1n\sum_{j\in[n]}g(\inpr{\wbf_k}{\zbf_j})\ge \frac1ng(\inpr{\wbf_k}{\zbf_i}) \ge \frac1ng\left(\ell^{-1}\left(\epsilon\right)\right)
  \end{equation}
  holds while $L(\wbf_k)$ is $\epsilon$-suboptimal, that is, $L(\wbf_k)>\epsilon$.

  Next, fix $\wbf_0=\zerobf$ and consider the case where $L(\wbf_k)>\epsilon$ holds for all $k\in[T-1]$.
  By \cref{lemma:loss_monotone}, we can use \cref{lemma:risk_eos}.
  By \cref{lemma:w_lb,lemma:risk_eos}, we can take $G_{\min}=g(\ell^{-1}(\epsilon))/n$ (noting \eqref{equation:gradient_lower_bound}) and have
  \begin{align*}
    \frac{\gamma\eta g\left(\ell^{-1}\left(\epsilon\right)\right)T}{n}
    &\le \inpr{\wbf_T}{\wbf_*} - \inpr{\wbf_0}{\wbf_*}
      && \text{(\cref{lemma:w_lb})} \\
    &= \inpr{\wbf_T}{\wbf_*}
      && \text{(with the choice of $\wbf_0=\zerobf$)} \\
    &\le \|\wbf_T\|
      && \text{(the Cauchy--Schwarz inequality with $\|\wbf_*\|=1$)} \\
    &\le \frac{4\sqrt{\rho(\gamma^2\eta T)} + \eta C_g}{\gamma}.
      && \text{(\cref{lemma:risk_eos})}
  \end{align*}
  By reorganizing and applying \cref{lemma:order_evaluation_gd}, we leverage the primal-dual relationship to have
  \[
    T \le \frac{n}{\gamma^2}\left(\frac{4\sqrt{\rho(\gamma^2\eta T)}}{\eta}+C_g\right)\cdot\frac{1}{g\left(\ell^{-1}\left(\epsilon\right)\right)}
    \le \frac{n}{\gamma^2}\left(\frac{4\sqrt{\rho(\gamma^2\eta T)}}{\eta}+C_g\right)\cdot\frac{1}{C_\phi\epsilon^{\alpha}},
  \]
  where $C_\phi$ is defined in \cref{lemma:order_evaluation_gd}.
  Therefore, $L(\wbf_k)$ is $\epsilon$-suboptimal after at most
  \[
    \frac{n}{C_\phi\gamma^2}\left(\frac{4\sqrt{\rho(\gamma^2\eta t)}}{\eta}+C_g\right)\epsilon^{-\alpha}
    \qquad (\eqdef T(\epsilon))
  \]
  iterations.
  That is, if $T>T(\epsilon)$, the gradient lower bound~\eqref{equation:gradient_lower_bound} must be violated at some $t\in[T]$,
  and for this $t$, we achieve $L(\wbf_t)\le\epsilon$.

  Finally, we verify that $C_\phi$ defined in \cref{lemma:order_evaluation_gd} matches \eqref{equation:exponent}.
  By introducing $\bar z$ as the dual of $\bar\mu$ such that
  \[
    \bar z=\phi'(\bar\mu) \quad \text{and} \quad \bar\mu=(\phi^*)'(\bar z),
  \]
  we have
  \[
    \begin{aligned}
      \bar\epsilon
      &= \ell(g^{-1}(\bar\mu))
        && \text{($g$ is invertible when $0<g(\cdot)<1$)} \\
      &= \phi^*(\bar z)
        && \text{($\bar\mu=(\phi^*)'(\bar z)=g(-\bar z)$ implies $g^{-1}(\bar\mu)=-\bar z$)} \\
      &= \bar\mu\phi'(\bar\mu)-\phi(\bar\mu),
    \end{aligned}
  \]
  where the invertibility of $g$ can be verified through the differentiability of $g$ as in \cref{lemma:order_evaluation_gd} (by relying upon $\phi''>0$ in \cref{assumption:fy_loss}),
  and the last identity follows by the definition of the convex conjugate.
  Plugging this into $C_\phi$ defined in \cref{lemma:order_evaluation_gd}, we see that it matches $C_\phi$ in \eqref{equation:exponent}.
  Thus, we have proven all statements.
\end{proof}
}

\section{Extension to the stochastic gradient descent}
\label{section:sgd}
We discuss the extension of \cref{theorem:gd} to the stochastic setup.
We consider the constant-stepsize online stochastic gradient descent (SGD) as follows:
\begin{equation}
  \label{equation:sgd} \tag{SGD}
  \wbf_{t+1} \defeq \wbf_t - \eta\nabla L_t(\wbf_t), \quad
  \text{where} \quad L_t(\wbf)\defeq\ell(\inpr{\wbf}{y_t\xbf_t}), \quad
  t\ge0,
\end{equation}
for a loss function $\ell\colon\Rbb\to\Rbb_{\ge0}$.
Here, $(\xbf_t,y_t)_{t\ge 0}$ are independently and identically distributed according to the following assumption.
\begin{assumption}
  \label{assumption:data_sgd}
  Assume the training data $(\xbf_t,y_t)_{t\ge0}$ are independent copies of $(\xbf,y)$ following a distribution such that
  \begin{enumerate}
    \item $\|\xbf\|\le 1$, and $y_i\in\set{\pm 1}$, almost surely;
    \item there is $\gamma>0$ and a unit vector $\wbf_*$ such that $\inpr{\wbf_*}{\zbf_t}\ge \gamma$ for $\zbf_t\defeq y_t\xbf_t$, almost surely.
  \end{enumerate}
\end{assumption}
\begin{proposition}
  \label{proposition:sgd}
  Suppose \cref{assumption:data_sgd} and consider \eqref{equation:sgd} with stepsize $\eta>0$ and $\wbf_0=\zerobf$
  under a Fenchel--Young loss $\ell$ satisfying \cref{assumption:loss}, and additionally having separation margin $m>0$.
  For arbitrary $\delta,\bar\epsilon\in(0,1)$, define $(\alpha,C_\phi)$ as in \cref{equation:exponent},
  for which we assume $\alpha,C_\phi\in(0,\infty)$.
  In addition, for arbitrary $\epsilon\in(0,\bar\epsilon)$, define
  \[
    M_{\eta,\gamma}\defeq\ell\left(-\frac{4m+\eta C_g}{\gamma}\right)
    \text{~~~and~~~}
    t^\circ\defeq\max\set{\frac{32M_{\eta,\gamma}^2\ln(1/\delta)}{\epsilon^2}, \frac{8M_{\eta,\gamma}\ln(1/\delta)}{\epsilon}}.
  \]
  Then, if we run \eqref{equation:sgd} with $T$ iterations such that
  \[
    T>N\cdot t^\circ, \text{~~~where~~~} N\defeq\frac{2^\alpha}{C_\phi\gamma^2}\left(\frac{4m}{\eta}+C_g\right)\epsilon^{-\alpha},
  \]
  then we have $\min_{t\in[T]}\E[L_t(\wbf_t)]\le \epsilon$
  with probability at least $(1-\delta)^N$.
\end{proposition}

Before proving \cref{proposition:sgd}, several auxiliary lemmas are presented.
Since the proof of \cref{proposition:sgd} closely follows \cref{theorem:gd}, the following \cref{lemma:w_lb_sgd,lemma:eos_sgd} are almost identical to the deterministic versions (\cref{lemma:w_lb,lemma:risk_eos}, respectively),
and hence we omit the proofs.
\begin{lemma}
  \label{lemma:w_lb_sgd}
  Suppose \cref{assumption:data_sgd} and consider \eqref{equation:sgd} with any stepsize $\eta>0$ under a Fenchel--Young loss that satisfies \cref{assumption:fy_loss}.
  Moreover, assume that there exists $k\in[0,t-1]$ such that we have a non-trivial lower bound $g(\inpr{\wbf_k}{\zbf_k})\ge g_{\min}>0$.
  Then, we have
  \[
    \gamma\eta g_{\min}\le \inpr{\wbf_t}{\wbf_*} - \inpr{\wbf_{0}}{\wbf_*}.
  \]
\end{lemma}
\begin{lemma}
  \label{lemma:eos_sgd}
  Suppose \cref{assumption:data_sgd} and consider \eqref{equation:sgd} with any stepsize $\eta>0$ under a convex and nondecreasing loss $\ell$ satisfying \cref{assumption:regular_loss,assumption:lipschitz_loss}.
  For every $t\ge 1$, we have
  \[
    \|\wbf_t\| \le \frac{4\sqrt{\rho(\gamma^2\eta t)} + \eta C_g}{\gamma}.
  \]
\end{lemma}
The following concentration result is an additional argument that we need in the stochastic case.
\begin{lemma}
  \label{lemma:concentration}
  Consider \eqref{equation:sgd} with any stepsize $\eta>0$ under a Fenchel--Young loss that satisfies \cref{assumption:fy_loss,assumption:lipschitz_loss}.
  Let $Z_0,\dots,Z_{t-1}$ be the independent copies of data $\zbf=y\xbf$ following \cref{assumption:data_sgd}.
  We introduce the filtration $\set{\Fcal_t}_{t\ge0}$, where $\Fcal_k$ is a $\sigma$-algebra defined on $Z_1,\dots,Z_k$,
  and let $H_k\defeq\ell(\inpr{\wbf_k}{Z_k})$ be a random variable, where $\wbf_k$ is an $\Fcal_{k-1}$-measurable random variable.
  We write $S_t\defeq\sum_{k=0}^{t-1}H_k$ for a random variable standing for the accumulated loss, and let $\overline{S}_t\defeq S_t/t$.
  Moreover, we introduce the following assumptions:
  \begin{itemize}
    \item \textbf{(Bounded mean)} $\mu_k\defeq\E[H_k|\Fcal_{k-1}]\le M$ for $k=0,\dots,t-1$.
    \item \textbf{(Bounded variance)} $\mathrm{Var}(H_k|\Fcal_{k-1})\le\sigma^2$ for $k=0,\dots,t-1$.
  \end{itemize}
  Then, for any $\epsilon>0$ and $\delta\in(0,1)$, if we have
  \begin{equation}
    \label{equation:concentration:condition}
    \frac{\E[S_t]}{t}>\epsilon \text{~~~and~~~}
    t\ge\max\set{\frac{32\sigma^2\ln(1/\delta)}{\epsilon^2}, \frac{8M\ln(1/\delta)}{\epsilon}},
  \end{equation}
  then we have $\overline{S}_t>\epsilon/2$ with probability at least $1-\delta$.
\end{lemma}
\begin{proof}
  First, we apply a type of the martingale inequality, \emph{Freedman's inequality}~\citep[Theorem~1.6]{Freedman1975AOP}.
  Since $\E[H_k-\mu_k|\Fcal_{k-1}]$ is the mean-zero martingale difference with the bounded variance $\sigma^2$, Freedman's inequality is applicable.
  Then, we have 
  \[
    \Pr\set{S_t\le \E[S_t]-u} \le \exp\left[-\frac{u^2}{2(Mu+t\sigma^2)}\right].
  \]
  Equivalently, we have the following inequality with probability at least $1-\delta$:
  \[
  \begin{aligned}
    S_t
    &> \E[S_t] - \sqrt{M^2\ln^2(1/\delta) + 2t\sigma^2\ln(1/\delta)} - M\ln(1/\delta) \\
    &> \E[S_t] - \sqrt{2t\sigma^2\ln(1/\delta)} - 2M\ln(1/\delta).
  \end{aligned}
  \]
  Dividing by $t$, we have
  \[
    \overline{S}_t
    > \frac{\E[S_t]}{t} - \sqrt{\frac{2\sigma^2\ln(1/\delta)}{t}} - \frac{2M\ln(1/\delta)}{t}
    > \epsilon - \frac{\epsilon}{4} - \frac{\epsilon}{4}
    = \frac{\epsilon}{2},
  \]
  where we used the conditions~\eqref{equation:concentration:condition} at the second inequality.
  Thus, the desired inequality is shown.
\end{proof}

Now we are ready to prove \cref{proposition:sgd}.
Overall, the proof consists of the perceptron argument and the concentration property.
\begin{proof}[Proof of \cref{proposition:sgd}]
  The first step is to establish the concentration property.
  To apply \cref{lemma:concentration}, we confirm the bounded moment conditions.
  For the mean $\E[H_k|\Fcal_{k-1}]=\E[\ell(\inpr{\wbf_k}{Z_k})|\Fcal_{k-1}]$,
  where $\Fcal_{k-1}$ is the $\sigma$-algebra defined on $\set{Z_l}_{l=1}^{k-1}$,
  we have
  \[
  \begin{aligned}
    \E[\ell(\inpr{\wbf_k}{Z_k})|\Fcal_{k-1}]
    &\le \E[\ell(-\|\wbf_k\|)|\Fcal_{k-1}] && \text{($\ell$ is nonincreasing and $\|Z_k\|\le 1$)} \\
    &\le \ell\left(-\frac{4m+\eta C_g}{\gamma}\right) && \text{(\cref{lemma:w_lb_sgd,lemma:rho_separation_margin})} \\
    &\eqdef M_{\eta,\gamma}.
  \end{aligned}
  \]
  For the variance $\mathrm{Var}(H_k|\Fcal_{k-1})$, we similarly have
  \[
  \begin{aligned}
    \mathrm{Var}(H_k|\Fcal_{k-1})
    &= \E[H_k^2|\Fcal_{k-1}] - \E[H_k|\Fcal_{k-1}]^2 \\
    &\le \E[H_k^2|\Fcal_{k-1}] \\
    &\le \ell\left(-\frac{4m+\eta C_g}{\gamma}\right)^2 \\
    &= M_{\eta,\gamma}^2.
  \end{aligned}
  \]
  Note that these moment bounds hold uniformly for any $k$.
  By plugging $M=M_{\eta,\gamma}$ and $\sigma^2=M_{\eta,\gamma}^2$ into \cref{lemma:concentration}, if we have
  \begin{equation}
    \label{equation:concentration_condition}
    \frac{1}{t^\circ}\sum_{k=t_0}^{t_0+t^\circ-1}\E\left[L_k(\wbf_k)\right]>\epsilon
    \text{~~~and~~~}
    t^\circ\ge\max\set{\frac{32M_{\eta,\gamma}^2\ln(1/\delta)}{\epsilon^2}, \frac{8M_{\eta,\gamma}\ln(1/\delta)}{\epsilon}},
  \end{equation}
  then we have
  \[
    \frac{1}{t^\circ}\sum_{k=t_0}^{t_0+t^\circ-1}L_k(\wbf_k) > \frac\epsilon2
    \text{~~~with probability at least $1-\delta$.}
  \]

  Let us use this concentration argument.
  Split the interval $[0,T]$ into length-$t^\circ$ sub-intervals (for $t^\circ$ satisfying \eqref{equation:concentration_condition}) such that
  \begin{equation}
    \label{equation:T_subinterval}
    [0,T] = \underbrace{[0,t^\circ-1]}_{\eqdef \Ical_1} \sqcup \underbrace{[t^\circ,2t^\circ-1]}_{\eqdef \Ical_2} \sqcup \underbrace{[2t^\circ,3t^\circ-1]}_{\eqdef \Ical_3} \dots \sqcup \underbrace{[(N-1)t^\circ,Nt^\circ-1]}_{\eqdef \Ical_N} \sqcup [Nt^\circ,T].
  \end{equation}
  Consider the scenario where $\E[L_k(\wbf_k)]>\epsilon$ holds for all $k\in[0,T]$,
  and focus on an arbitrary sub-interval $\Ical_l$.
  Since we have $\frac{1}{t^\circ}\sum_{k\in\Ical_l}\E[L_k(\wbf_k)]>\epsilon$,
  the concentration argument implies that $\frac{1}{t^\circ}\sum_{k\in\Ical_l}L_k(\wbf_k)>\epsilon/2$ with probability at least $1-\delta$.
  This further indicates the high-probability existence of $k_l\in\Ical_l$ such that $L_{k_l}(\wbf_{k_l})>\epsilon/2$.
  In this case, we have $g(\inpr{\wbf_{k_l}}{\zbf_{k_l}})\ge g(\ell^{-1}(\epsilon/2))$ for this specific $k_l\in\Ical_l$,
  which can be seen in the same manner as the proof of \cref{theorem:gd}.
  Since this concentration argument does not depend on the sub-interval choice $\Ical_l$,
  there exists a set of indices $\set{k_l}_{l\in[N]}$ such that each of $g(\inpr{\wbf_{k_l}}{\zbf_{k_l}})\ge g(\ell^{-1}(\epsilon/2))$ holds with probability at least $1-\delta$.

  Next, we invoke the perceptron argument.
  By combining \cref{lemma:w_lb_sgd,lemma:eos_sgd} with the choice $g_{\min}=g(\ell^{-1}(\epsilon/2))$, we have
  \[
    \gamma\eta g\left(\ell^{-1}\left(\frac\epsilon2\right)\right)\cdot N \le \inpr{\wbf_T}{\wbf_*}
    \le \|\wbf_T\|
    \le \frac{4\sqrt{\rho(\gamma^2\eta T)}+\eta C_g}{\gamma}
    \le \frac{4m+\eta C_g}{\gamma},
  \]
  with probability at least $(1-\delta)^N$,
  where we additionally used \cref{lemma:rho_separation_margin} at the last inequality.
  By applying \cref{lemma:order_evaluation_gd} at the left-most side,
  we have
  \[
    \gamma\eta C_\phi\cdot\left(\frac\epsilon2\right)^{\alpha}\cdot N \le \frac{4m+\eta C_g}{\gamma},
  \]
  which implies
  \[
    N \le \frac{2^\alpha}{C_\phi\gamma^2}\left(\frac{4m}{\eta}+C_g\right) \epsilon^{-\alpha} \quad (\eqdef N(\epsilon)).
  \]
  Therefore, if $N>N(\epsilon)$ (or $T>N(\epsilon)\cdot t^\circ$), with probability at least $(1-\delta)^N$, we have $\E[L_k(\wbf_k)]\le\epsilon$ for some $k\in[0,T]$.

  Finally, we need verify that $C_\phi$ defined in \cref{lemma:order_evaluation_gd} matches \eqref{equation:exponent},
  but we skip it because it can be confirmed in the same manner as in the proof of \cref{theorem:gd}.
\end{proof}

\subsection{Comparison between GD and SGD}
Whereas the iteration complexity for GD given by \cref{corollary:gd_separation_margin} is $T=\Omega(\epsilon^{-\alpha})$,
the iteration complexity for SGD given by \cref{proposition:sgd} is $T=\Omega(\epsilon^{-(\alpha+2)})$.
Hence, the SGD rate is significantly slower than the GD case.
This deterioration is because we can observe a non-trivial lower bound $g(\inpr{\wbf_k}{\zbf_k})$ after every $t^\circ=\Omega(\epsilon^{-2})$ steps.
As in the GD case, we need $N=\Omega(\epsilon^{-\alpha})$ such non-trivial lower bounds, and hence the total iteration number amounts to $Nt^\circ=\Omega(\epsilon^{-(\alpha+2)})$.
While this apparently looks a big bottleneck, note that the GD rate is $T\gtrsim n\epsilon^{-\alpha}$ if we explicitly write the $n$-dependency.
Since \eqref{equation:sgd} consumes only one fresh sample at every update (while \eqref{equation:gd} consumes $n$ samples),
the extra complexity $N=\Omega(\epsilon^{-2})$ appearing in the SGD case compensates for this gap of sample sizes.
The GD/SGD rates are comparable in this sense.

\section{Phase transition of large-stepsize GD}
\label{proof:gd_self_bounding}
We recap \citet{Wu2024COLT}, who shows the existence of the phase transition from the EoS to stable phases.
\begin{assumption}
  \label{assumption2:loss}
  Consider a loss $\ell\in\Ccal^1(\Rbb)$ that is convex, nonincreasing, and $\ell(+\infty)=0$.
  \begin{assumpenum}
    \item \label{assumption:self_bounding} \textbf{Self-bounding property.}
    For some $C_\beta>0$, $g(\cdot)\le C_\beta\ell(\cdot)$ and
    \[
      \ell(z)\le\ell(x)+\ell'(x)(z-x)+C_\beta g(x)(z-x)^2 \quad \text{for $z$ and $x$ such that $|z-x|<1$.}
    \]

    \item \label{assumption:exponential_tail} \textbf{Exponential tail.}
    There is a constant $C_e>0$ such that $\ell(z)\le C_eg(z)$ for $z\ge0$.
  \end{assumpenum}
\end{assumption}
\begin{theorem}[{\cite{Wu2024COLT}}]
  \label{theorem:gd_self_bounding}
  Consider \eqref{equation:gd} with stepsize $\eta>0$ and initialization $\wbf_0=\zerobf$ under a loss $\ell$ satisfying \cref{assumption:regular_loss,assumption:lipschitz_loss}, and~\ref{assumption:self_bounding}.
  Let $T$ be the maximum number of steps.
  Then, we have the following:
  \begin{itemize}
    \item \textbf{The EoS phase.} For every $t>0$, we have
    \[
      \frac1t\sum_{k=0}^{t-1}L(\wbf_k) \le \frac{[6\sqrt{\rho(\gamma^2\eta t)} + \eta C_g]^2}{8\gamma^2\eta t}.
    \]

    \item \textbf{The stable phase.} If $s<T$ is such that
    \begin{equation}
      \label{equation:enter_stable_phase}
      L(\wbf_s) \le \min\set{\frac{1}{4C_\beta^2\eta}, \frac{\ell(0)}{n}},
    \end{equation}
    then \eqref{equation:gd} is in the stable phase, that is, $(L(\wbf_t))_{t=s}^T$ decreases monotonically, and moreover,
    \[
      L(\wbf_t) \le 5\frac{\rho(\gamma^2\eta(t-s))}{\gamma^2\eta(t-s)}, \quad t\in(s,T].
    \]

    \item \textbf{Phase transition time.} There exists a constant $C_1>0$ that only depends on $C_g$, $C_\beta$, and $\ell(0)$ such that the following holds. Let
    \[
      \tau\defeq\frac{1}{\gamma^2}\max\set{\frac{\psi^{-1}(C_1(\eta+n))}{\eta}, C_1(\eta+n)\eta}, \quad \text{where} \quad \psi(\lambda)\defeq\frac{\lambda}{\rho(\lambda)}.
    \]
    If $\tau\le T$, \eqref{equation:enter_stable_phase} holds for some $s\le\tau$.
    Moreover, if $\ell$ additionally satisfies \cref{assumption:exponential_tail} and $\eta\ge1$, there exists $C_2>0$ that depends on $C_e$, $C_g$, $C_\beta$, $\ell(0)$, and $n$ such that $\tau$ is improved as follows:
    \[
      \tau\defeq\frac{C_2}{\gamma^2}\max\set{\eta, n}.
    \]
  \end{itemize}
\end{theorem}
The proof consists of \cref{lemma:risk_eos} (the EoS phase), \cref{lemma:stable_phase_convergence} (the stable phase), and \cref{lemma:phase_transition,lemma:phase_transition_exponential_tail} (phase transition time), respectively.
Most of the results in this section have already been provided in \citet{Wu2024COLT}.
We restate the statements and proofs here to make the paper self-contained, and moreover,
simplify the statements from the NTK setup to the linear-model case
to highlight the essential structures.

\begin{lemma}[EoS phase]
  \label{lemma:risk_eos}
  Suppose \cref{assumption:data}
  and consider a convex and nonincreasing loss $\ell$ satisfying \cref{assumption:regular_loss,assumption:lipschitz_loss}.
  For every $t\ge1$, we have
  \[
    \frac1t\sum_{k=0}^{t-1}L(\wbf_k) \le \frac{[6\sqrt{\rho(\gamma^2\eta t)} + \eta C_g]^2}{8\gamma^2\eta t},
  \]
  and
  \[
    \|\wbf_t\| \le \frac{4\sqrt{\rho(\gamma^2\eta t)} + \eta C_g}{\gamma}.
  \]
\end{lemma}

\begin{proof}
  By invoking \cref{lemma:split_optimization} with the choice of $\theta$
  \[
    \theta=\frac{\sqrt{\rho(\gamma^2\eta t)}}{\gamma},
  \]
  we have
  \[
    \begin{aligned}
      \frac{\|\wbf_t-\ubf\|^2}{2\eta t} + \frac1t\sum_{k=0}^{t-1}L(\wbf_k)
      &\le \ell(\gamma\theta) + \frac{1}{2\eta t}\left(\theta+\frac{\eta C_g}{2\gamma}\right)^2 \\
      &\le \frac{\rho(\gamma^2\eta t)}{\gamma^2\eta t} + \frac{1}{2\eta t}\left(\theta+\frac{\eta C_g}{2\gamma}\right)^2,
        && \text{(\cref{lemma:loss_bound})} \\
    \end{aligned}
  \]
  which implies that
  \[
    \begin{aligned}
      \|\wbf_t\|
      &\le \|\wbf_t - \ubf\| + \|\ubf\| \\
      &\le \sqrt{\frac{2\rho(\gamma^2\eta t)}{\gamma^2} + \left(\theta+\frac{\eta C_g}{2\gamma}\right)^2} + \left(\theta+\frac{\eta C_g}{2\gamma}\right) \\
      &\le \frac{\sqrt{2\rho(\gamma^2\eta t)}}{\gamma} + 2\left(\theta+\frac{\eta C_g}{2\gamma}\right)
        && \text{($\sqrt{a+b}\le\sqrt{a}+\sqrt{b}$)} \\
      &\le \frac{4\sqrt{\rho(\gamma^2\eta t)} + \eta C_g}{\gamma},
    \end{aligned}
  \]
  and
  \[
    \begin{aligned}
      \frac1t\sum_{k=0}^{t-1}L(\wbf_k)
      &\le \frac{\rho(\gamma^2\eta t)}{\gamma^2\eta t} + \frac{1}{2\eta t}\left(\theta+\frac{\eta C_g}{2\gamma}\right)^2 \\
      &= \frac{\rho(\gamma^2\eta t)}{\gamma^2\eta t} + \frac{(2\sqrt{\rho(\gamma^2\eta t)} + \eta C_g)^2}{8\gamma^2\eta t} \\
      &\le \frac{[2(1+\sqrt{2})\sqrt{\rho(\gamma^2\eta t)} + \eta C_g]^2}{8\gamma^2\eta t}
        && \text{($a^2+b^2\le(a+b)^2$ for $a,b\ge0$)} \\
      &\le \frac{[6\sqrt{\rho(\gamma^2\eta t)} + \eta C_g]^2}{8\gamma^2\eta t}.
    \end{aligned}
  \]
  Thus, the proof is completed.
\end{proof}

\begin{lemma}
  \label{lemma:gradient_potential_eos}
  Suppose \cref{assumption:data}
  and consider a convex and nonincreasing loss $\ell$ satisfying \cref{assumption:regular_loss,assumption:lipschitz_loss}.
  Then, we have
  \[
    \frac1t\sum_{k=0}^{t-1}G(\wbf_k) \le \frac{4\sqrt{\rho(\gamma^2\eta t)} + \eta C_g}{\gamma^2\eta t}, \quad t \le T,
  \]
  where $G(\wbf)$ is defined in \eqref{equation:gradient_potential}.
\end{lemma}

\begin{proof}
  By the perceptron argument~\cite{Novikoff1962}, we have
  \[
    \begin{aligned}
      \inpr{\wbf_{t+1}}{\wbf_*}
      &= \inpr{\wbf_t}{\wbf_*} - \eta\inpr{\nabla L(\wbf_t)}{\wbf_*} \\
      &= \inpr{\wbf_t}{\wbf_*} -\frac\eta n\sum_{i=1}^n\ell'(\inpr{\wbf_t}{\zbf_i})\inpr{\wbf_*}{\zbf_i} \\
      &\ge \inpr{\wbf_t}{\wbf_*} - \frac{\gamma\eta}{n}\sum_{i=1}^n\ell'(\inpr{\wbf_t}{\zbf_i})
        && \text{(\cref{assumption:data} and note $-\ell'(\cdot)\ge0$)} \\
      &= \inpr{\wbf_t}{\wbf_*} - \gamma\eta G(\wbf_t).
    \end{aligned}
  \]
  Telescoping the sum, we have
  \[
    \frac1t\sum_{k=0}^{t-1}G(\wbf_k) \le \frac{\inpr{\wbf_t}{\wbf_*} - \inpr{\wbf_0}{\wbf_*}}{\gamma\eta t}
    \le \frac{\|\wbf_t\|}{\gamma\eta t}
    \le \frac{4\sqrt{\rho(\gamma^2\eta t)} + \eta C_g}{\gamma^2\eta t},
  \]
  where the last inequality is due to the parameter bound in \cref{lemma:risk_eos}.
\end{proof}

\begin{lemma}[Modified descent lemma]
  \label{lemma:stable_phase}
  Consider a loss satisfying \ref{assumption:self_bounding}.
  Suppose there exists $s<T$ such that
  \[
    L(\wbf_s)\le\frac{1}{4C_\beta^2\eta},
  \]
  then for every $t\in[s,T]$ we have
  \begin{enumerate}
    \item $L(\wbf_t)\le1/(4C_\beta^2\eta)$ and $G(\wbf_t)\le1/(4C_\beta\eta)$,
    \item $L(\wbf_{t+1})\le L(\wbf_t)-\frac{3\eta}{4}\|\nabla L(\wbf_t)\|^2\le L(\wbf_t)$,
  \end{enumerate}
  where $G(\wbf)$ is defined in \eqref{equation:gradient_potential}.
\end{lemma}

\begin{proof}
  We first show that Claim~1 implies Claim~2.
  By \cref{assumption:self_bounding}, we have
  \[
    \begin{aligned}
      \ell(\inpr{\wbf_{t+1}}{\zbf_i})
      &\le \ell(\inpr{\wbf_t}{\zbf_i}) + \ell'(\inpr{\wbf_t}{\zbf_i})\inpr{\wbf_{t+1}-\wbf_t}{\zbf_i} + C_\beta g(\inpr{\wbf_t}{\zbf_i})\inpr{\wbf_{t+1}-\wbf_t}{\zbf_i}^2 \\
      &\le \ell(\inpr{\wbf_t}{\zbf_i}) + \ell'(\inpr{\wbf_t}{\zbf_i})\inpr{\wbf_{t+1}-\wbf_t}{\zbf_i} + C_\beta g(\inpr{\wbf_t}{\zbf_i})\|\wbf_{t+1}-\wbf_t\|^2.
    \end{aligned}
  \]
  Taking average over $i\in[n]$, we get
  \[
    \begin{aligned}
      L(\wbf_{t+1})
      &\le L(\wbf_t) + \inpr{\nabla L(\wbf_t)}{\wbf_{t+1}-\wbf_t} + C_\beta G(\wbf_t)\|\wbf_{t+1}-\wbf_t\|^2 \\
      &= L(\wbf_t) - \eta\|\nabla L(\wbf_t)\|^2 + C_\beta\eta^2 G(\wbf_t)\|\nabla L(\wbf_t)\|^2 \\
      &\le L(\wbf_t) - \eta\|\nabla L(\wbf_t)\|^2 + \frac{\eta}{4}\|\nabla L(\wbf_t)\|^2
        && \text{(Claim 1)} \\
      &= L(\wbf_t) - \frac{3\eta}{4}\|\nabla L(\wbf_t)\|^2,
    \end{aligned}
  \]
  which verifies Claim~2.

  Next, we prove Claim~1 by induction.
  The base case $t=s$ holds by \cref{assumption:self_bounding} as follows:
  \begin{equation}
    \label{equation:self_bounding_potential}
    G(\wbf_s) = \frac1n\sum_{i=1}^ng(\inpr{\wbf_s}{\zbf_i}) \le \frac1n\sum_{i=1}^nC_\beta\ell(\inpr{\wbf_s}{\zbf_i}) = C_\beta L(\wbf_s) \le \frac{1}{4C_\beta\eta}.
  \end{equation}
  To prove the step case, we suppose $L(\wbf_k)\le1/(4C_\beta^2\eta)$ and $G(\wbf_k)\le1/(4C_\beta\eta)$ for $k=s,s+1,\dots,t$ and prove them for $k=t+1$.
  Since Claim~1 implies Claim~2, we have
  \[
    L(\wbf_{t+1})\le L(\wbf_t) \le \dots \le L(\wbf_s) \le \frac{1}{4C_\beta^2\eta}.
  \]
  Since $G(\wbf_{t+1})\le C_\beta L(\wbf_{t+1})$ holds as in \eqref{equation:self_bounding_potential}, we have
  \[
    G(\wbf_{t+1}) \le C_\beta L(\wbf_{t+1}) \le \frac{1}{4C_\beta\eta}.
  \]
  Thus, the step case is shown, and all claims are proven.
\end{proof}

\begin{lemma}
  \label{lemma:correct_classification}
  Consider a nonincreasing and nonnegative loss $\ell$.
  For every $\wbf$ such that
  \[
    L(\wbf) \le \frac{\ell(0)}{n},
  \]
  we have $y_i\inpr{\wbf}{\xbf_i}\ge0$ for $i\in[n]$.
\end{lemma}

\begin{proof}
  See \cite[Lemma 31]{Wu2024COLT}.
\end{proof}

\begin{lemma}[Stable phase]
  \label{lemma:stable_phase_convergence}
  Consider a nonincreasing loss $\ell$ satisfying \cref{assumption:regular_loss,assumption:self_bounding}.
  Suppose there exists $s<T$ such that
  \[
    L(\wbf_s)\le\min\set{\frac{1}{4C_\beta^2\eta}, \frac{\ell(0)}{n}}.
  \]
  Then, for every $t\in[0,T-s]$, we have
  \[
    L(\wbf_{s+t})\le 5\frac{\rho(\gamma^2\eta t)}{\gamma^2\eta t}.
  \]
\end{lemma}

\begin{proof}
  By the lemma assumption, we can apply \cref{lemma:stable_phase} for $s$ onwards.
  Therefore, we have for $k\ge0$,
  \begin{equation}
    \label{equation:proof:supp:1}
    \eta\|\nabla L(\wbf_{s+k})\|^2 \le \frac43[L(\wbf_{s+k}) - L(\wbf_{s+k+1})] \le \frac43L(\wbf_{s+k}).
  \end{equation}
  Choose a comparator centered at $\wbf_s$,
  \[
    \ubf \defeq \wbf_s + \theta\wbf_*, \quad
    \theta \defeq \frac{\sqrt{\rho(\eta^2\gamma t)}}{\gamma}.
  \]
  For $k\le t-1$, we have
  \[
    \begin{aligned}
      \|\wbf_{s+k+1}-\ubf\|^2
      &= \|\wbf_{s+k}-\ubf\|^2 + 2\eta\inpr{\nabla L(\wbf_{s+k})}{\ubf-\wbf_{s+k}} + \eta^2\|\nabla L(\wbf_{s+k})\|^2 \\
      &\le \|\wbf_{s+k}-\ubf\|^2 + 2\eta\inpr{\nabla L(\wbf_{s+k})}{\ubf-\wbf_{s+k}} + \frac43\eta L(\wbf_{s+k}).
        \qquad \text{(by \eqref{equation:proof:supp:1})}
    \end{aligned}
  \]
  Following the same derivation of \eqref{equation:proof:supp:2}, we can bound the second term as follows:
  \[
    \inpr{\nabla L(\wbf_{s+k})}{\ubf-\wbf_{s+k}}
    \le \frac1n\sum_{i=1}^n\ell(\theta\gamma + \inpr{\wbf_s}{\zbf_i}) - L(\wbf_{s+k}).
  \]
  The assumption $L(\wbf_s)\le\ell(0)/n$ allows us to apply \cref{lemma:correct_classification}, so $\inpr{\wbf_s}{\zbf_i}\ge0$ and thus
  \[
    \ell(\theta\gamma+\inpr{\wbf_s}{\zbf_i}) 
    \le \ell(\theta\gamma)
    = \ell(\sqrt{\rho(\gamma^2\eta t)}),
  \]
  where $\ell$ is nonincreasing due to the lemma assumption.
  Consequently, we can control the second term by
  \[
    \inpr{\nabla L(\wbf_{s+k})}{\ubf-\wbf_{s+k}} \le \ell(\sqrt{\rho(\gamma^2\eta t)}) - L(\wbf_{s+k}).
  \]
  Plugging this back, we get
  \[
    \begin{aligned}
      \|\wbf_{s+k+1}-\ubf\|^2
      &\le \|\wbf_{s+k}-\ubf\|^2 + 2\eta[\ell(\sqrt{\rho(\gamma^2\eta t)}) - L(\wbf_{s+k})] + \frac43\eta L(\wbf_{s+k}) \\
      &\le \|\wbf_{s+k}-\ubf\|^2 + 2\eta\ell(\sqrt{\rho(\gamma^2\eta t)}) - \frac23\eta L(\wbf_{s+k}).
    \end{aligned}
  \]
  Telescoping the sum from $0$ to $t-1$ and rearranging, we get
  \[
    \begin{aligned}
      \frac{3\|\wbf_{s+t}-\ubf\|^2}{2\eta t} + \frac1t\sum_{k=0}^{t-1}L(\wbf_{s+k})
      &\le 3\ell(\sqrt{\rho(\gamma^2\eta t)}) + \frac{3\|\wbf_s-\ubf\|^2}{2\eta t} \\
      &\le 3\frac{\rho(\gamma^2\eta t)}{\gamma^2\eta t} + \frac{3\|\wbf_s-\ubf\|^2}{2\eta t}.
        && \text{(\cref{lemma:loss_bound})}
    \end{aligned}
  \]
  Finally, we can show the claims as follows:
  \[
    \begin{aligned}
      \frac1t\sum_{k=0}^{t-1}L(\wbf_{s+k})
      &\le 3\frac{\rho(\gamma^2\eta t)}{\gamma^2\eta t} + \frac{3\|\wbf_s-\ubf\|^2}{2\eta t} \\
      &= \frac92\frac{\rho(\gamma^2\eta t)}{\gamma^2\eta t} \\
      &\le 5\frac{\rho(\gamma^2\eta t)}{\gamma^2\eta t}.
    \end{aligned}
  \]
  By \cref{lemma:stable_phase}, $L(\wbf_t)$ is nonincreasing for $t\ge s$, and thus we have
  \[
    L(\wbf_{s+t}) \le \frac1t\sum_{k=0}^{t-1}L(\wbf_{s+k})
    \le 5\frac{\rho(\gamma^2\eta t)}{\gamma^2\eta t}.
  \]
\end{proof}

\begin{lemma}[Phase transition]
  \label{lemma:phase_transition}
  Suppose \cref{assumption:data}
  and consider a convex and nonincreasing loss $\ell$ that satisfies \cref{assumption:regular_loss,assumption:lipschitz_loss}.
  Define
  \[
    \psi(\lambda)=\frac{\lambda}{\rho(\lambda)}, \quad \lambda>0.
  \]
  Then, there is $C>0$ depending on $C_g$, $C_\beta$, and $\ell(0)$ such that the following holds.
  Let
  \[
    \tau\defeq\frac{1}{\gamma^2}\max\set{\frac{\psi^{-1}(C(\eta+n))}{\eta}, C(\eta+n)\eta}.
  \]
  If $\tau<T$, then there exists $s\in[0,\tau]$ such that
  \[
    L(\wbf_s)\le\min\set{\frac{1}{4C_\beta^2\eta}, \frac{\ell(0)}{n}}.
  \]
\end{lemma}

\begin{proof}
  Applying \cref{lemma:risk_eos} with $t=\tau$, we have
  \[
    \frac1\tau\sum_{k=0}^{\tau-1}L(\wbf_k)
    \le \frac{[6\sqrt{\rho(\gamma^2\eta\tau)} + \eta C_g]^2}{8\gamma^2\eta\tau}
    = \left[\frac{3}{\sqrt2}\sqrt{\frac{\rho(\gamma^2\eta\tau)}{\gamma^2\eta\tau}} + \frac{\sqrt2}{4}\frac{\eta C_g}{\sqrt{\gamma^2\eta\tau}}\right]^2.
  \]
  Choose $\tau$ such that
  \[
    \gamma^2\eta\tau
    \ge \max\set{
      \psi^{-1}\left(18[4C_\beta^2\eta+n/\ell(0)]\right),
      \frac12(\eta C_g)^2(4C_\beta^2\eta+n/\ell(0))
    }.
  \]
  It is clear that
  \[
    \frac{1}{\psi(\lambda)}=\frac{\rho(\lambda)}{\lambda}=\min_z\ell(z)+\frac{z^2}{\lambda}
  \]
  is a decreasing function.
  Then, we have
  \[
    \frac{3}{\sqrt2}\sqrt{\frac{\rho(\gamma^2\eta\tau)}{\gamma^2\eta\tau}}
    = \frac{3}{\sqrt2}\sqrt{\frac{1}{\psi(\gamma^2\eta\tau)}}
    \le \frac{3}{\sqrt2}\sqrt{\frac{1}{18[4C_\beta^2\eta+n/\ell(0)]}}
    = \frac12\frac{1}{\sqrt{4C_\beta^2\eta+n/\ell(0)}}
  \]
  and
  \[
    \frac{\sqrt2}{4}\frac{\eta C_g}{\sqrt{\gamma^2\eta\tau}}
    \le \frac12\frac{1}{\sqrt{4C_\beta^2\eta+n/\ell(0)}}.
  \]
  These two inequalities together imply that
  \[
    \frac1\tau\sum_{k=0}^{\tau-1}L(\wbf_k)
    \le \frac{1}{4C_\beta^2\eta+n/\ell(0)}
    \le \min\set{\frac{1}{4C_\beta^2\eta}, \frac{\ell(0)}{n}},
  \]
  which implies that there exists $s\le\tau$ for $L(\wbf_s)$ satisfies the right-hand side bound.
\end{proof}

\begin{lemma}[Phase transition time under exponential tail]
  \label{lemma:phase_transition_exponential_tail}
  Suppose \cref{assumption:data}
  and consider a nonincreasing loss $\ell$ satisfying \cref{assumption:regular_loss,assumption:lipschitz_loss}, and~\ref{assumption:exponential_tail}.
  Furthermore, assume $\eta\ge1$.
  Then, there exists $C>0$ depending on $C_e$, $C_g$, $C_\beta$, $\ell(0)$, and $n$ such that the following holds.
  Let
  \[
    \tau\defeq\frac{C}{\gamma^2}\max\set{\eta, n\ln n}.
  \]
  If $\tau\le T$, then there exists $s\in[0,\tau]$ such that
  \[
    L(\wbf_s)\le\min\set{\frac{1}{4C_\beta^2\eta}, \frac{\ell(0)}{n}}.
  \]
\end{lemma}

\begin{proof}
  Under \cref{assumption:exponential_tail}, we have
  \[
    \ell(z)\le C_eg(z) = -C_e\ell'(z), \quad \text{for $z\ge 0$,}
  \]
  which implies
  \[
    \frac{\ell'(z)}{\ell(z)}\le -C_e^{-1}, \quad \text{for $z\ge0$.}
  \]
  Integrating both sides, we get
  \[
    \ln\ell(z)\le\ln\ell(0)+\int_0^z\frac{\ell'(\zeta)}{\ell(\zeta)}\rd{\zeta}
    \le\ln\ell(0)-C_e^{-1}z, \quad \text{for $z\ge0$,}
  \]
  which implies
  \[
    \ell(z)\le\ell(0)\exp(-C_e^{-1}z), \quad \text{for $z\ge0$.}
  \]
  Using the exponential tail property, we have
  \[
    \rho(\lambda)
    =\min_{z\in\Rbb}\lambda\ell(z)+z^2
    \le \lambda\ell(C_e\ln(\lambda)) + C_e^2\ln^2(\lambda)
    \le \ell(0) + C_e^2\ln^2(\lambda).
  \]
  Applying \cref{lemma:gradient_potential_eos}, we have
  \begin{align*}
    \frac1\tau\sum_{k=0}^{\tau-1}G(\wbf_k)
    &\le \frac{4\sqrt{\rho(\gamma^2\eta\tau)}+\eta C_g}{\gamma^2\eta\tau} \\
    &\le \frac{4\sqrt{\ell(0)+C_e^2\ln^2(\gamma^2\eta\tau)}+\eta C_g}{\gamma^2\eta\tau} \\
    &\le \frac{4\sqrt{\ell(0)}+4C_e\ln(\gamma^2\eta\tau)+\eta C_g}{\gamma^2\eta\tau}
      && (\sqrt{a+b}\le\sqrt{a}+\sqrt{b}) \\
    &\le \frac{4C_e}{\eta}\frac{\ln(\gamma^2\tau)}{\gamma^2\tau} + \frac{C_g+4C_e}{\gamma^2\tau} + \frac{4\sqrt{\ell(0)}}{\eta}\frac{1}{\gamma^2\tau}.
  \end{align*}
  Here, we take $C>0$ depending on $C_e$, $C_g$, $C_\beta$, $\ell(0)$, and additionally $n$ such that
  \[
    \gamma^2\tau\ge C\max\set{\eta, n}
  \]
  and
  \[
    \frac{\ln C}{C}\le \frac{\min\set{\frac{1}{4C_eC_\beta^2}, \frac{\ell(0)}{C_e}}}{4C_e(1+\ln n)+C_g+4C_e+4\sqrt{\ell(0)}}.
  \]
  This choice is possible with sufficiently large $C\ge e$ because $(\ln C)/C$ is strictly decreasing in $C\ge e$ toward zero.
  Such $C$ enables us to have
  \begin{align*}
    \frac1\tau&\sum_{k=0}^{\tau-1}G(\wbf_k) \\
    &\le \frac{1}{C\max\set{\eta, n}}\left[\frac{4C_e}{\eta}(\ln C+\ln\max\set{\eta,n}) + C_g+4C_e + \frac{4\sqrt{\ell(0)}}{\eta}\right] \\
    &\le \frac{4C_e(\ln C+\ln n) + C_g+4C_e + 4\sqrt{\ell(0)}}{C\max\set{\eta, n}}
      && (\eta\ge1) \\
    &\le \frac{\ln C}{C}\frac{4C_e(1+\ln n) + C_g+4C_e + 4\sqrt{\ell(0)}}{\max\set{\eta, n}}
      && (\ln C\ge1) \\
    &\le \frac{\min\set{\frac{1}{4C_eC_\beta^2}, \frac{\ell(0)}{C_e}}}{\max\set{\eta,n}} \\
    &\le \min\set{\frac{1}{4C_eC_\beta^2\eta}, \frac{\ell(0)}{C_en}}.
  \end{align*}
  From this we have some $s\le\tau$ such that
  \[
    G(\wbf_s)\le\min\set{\frac{1}{4C_eC_\beta^2\eta}, \frac{\ell(0)}{C_en}}.
  \]
  This ensures that for every $i\in[n]$,
  \[
    \frac1ng(\inpr{\wbf_s}{\zbf_i})\le G(\wbf_s)
    =\frac1n\sum_{i=1}^ng(\inpr{\wbf_s}{\zbf_i})
    \le\frac{\ell(0)}{C_en}
    \le\frac{g(0)}{n},
  \]
  where the last inequality is due to \cref{assumption:exponential_tail}.
  The above implies $\inpr{\wbf_s}{\zbf_i}\ge0$ since $g(\cdot)$ is nonincreasing.
  Thus, we can apply \cref{assumption:exponential_tail} for $\inpr{\wbf_s}{\zbf_i}$ and get
  \[
    \ell(\inpr{\wbf_s}{\zbf_i})\le C_eg(\inpr{\wbf_s}{\zbf_i}).
  \]
  Taking an average over $i\in[n]$, we have
  \[
    L(\wbf_s)\le C_eG(\wbf_s).
  \]
  We complete the proof by plugging in the upper bound on $G(\wbf_s)$.
\end{proof}

\section{Separation margin and self-bounding property}
\label{appendix:separation_margin}
In this section, we discuss the relationship between separation margin and the self-bounding property.
First, we show that a loss function does not have separation margin if it satisfies the self-bounding property.

\begin{proposition}
  \label{proposition:no_separation_margin}
  Consider a loss $\ell\colon\Rbb\to\Rbb$ that is continuously differentiable and nonincreasing, and satisfies $\ell(z_0)>0$ for some $z_0\in\Rbb$.
  If $\ell$ satisfies \cref{assumption:self_bounding}, then $\ell$ does not have separation margin.
\end{proposition}
\begin{proof}
  Choose any $\epsilon\in(0,1/C_\beta)$.
  The convexity of $\ell$ implies that
  \[
    g(z) = -\ell'(z) \ge \frac{\ell(z) - \ell(z+\epsilon)}{\epsilon}
    \quad \text{for any $z\in\Rbb$.}
  \]
  By the self-bounding property (\cref{assumption:self_bounding}), we further have
  \[
    \frac{\ell(z) - \ell(z+\epsilon)}{\epsilon} \le g(z) \le C_\beta\ell(z).
  \]
  Solving this, we have
  \[
    \ell(z+\epsilon) \ge (1-C_\beta\epsilon)\ell(z).
  \]
  Thus, if $\ell(z)>0$ holds, we additionally have $\ell(z+\epsilon)>0$ for $\epsilon\in(0,1/C_\beta)$,
  and we conclude that $\ell$ cannot have separation margin because $\ell>0$ holds on the entire $\Rbb$.
\end{proof}

Next, we argue that the converse of \cref{proposition:no_separation_margin} does not hold,
that is, even if a loss $\ell$ does not have separation margin, it does not always imply that $\ell$ satisfies the self-bounding property.
A counterexample is a Fenchel--Young loss generated by the following potential function:
\[
  \phi(\mu)=\int_0^\mu\Phi^{-1}(p)\rd{p},
  \quad \text{where $\Phi$ is the standard normal CDF}
  \;\; \Phi(x)\defeq\frac12\left[1+\mathrm{erf}\left(\frac{x}{\sqrt2}\right)\right]
\]
and $\mathrm{erf}$ is the error function.
The generated Fenchel--Young loss is relevant to the probit model~\cite{McCullagh1989} because $\phi'$ is nothing else but the probit link function prevailing in generalized linear models.
Hence, we call the generated Fenchel--Young loss the \emph{probit Fenchel--Young loss} for convenience.
We can have a concise form of the probit Fenchel-Young loss:
\begin{align*}
  \ell(z)
  &= \phi^*(-z) \\
  &= \int_{-\infty}^{-z}(\phi')^{-1}(\zeta)\rd{\zeta} \\
  &= \int_{-\infty}^{-z}\Phi(\zeta)\rd{\zeta} \\
  &= [\zeta\Phi(\zeta)+\Phi'(\zeta)]_{-\infty}^{-z} \\
  &= -z\Phi(-z)+\Phi'(-z).
\end{align*}
The probit Fenchel--Young loss does not have separation margin because $\phi'(\mu)=\Phi^{-1}(\mu)\to-\infty$ as $\mu\downarrow0$ (see \cref{proposition:separation_margin}).
However, it does not satisfy the self-bounding property.
To see this, we have
\begin{align*}
  \frac{g(z)}{\ell(z)}
  &= -\frac{\ell'(z)}{\ell(z)} \\
  &= -\frac{-\Phi(-z)}{-z\Phi(-z)+\Phi'(-z)} \\
  &= \frac{\Phi(\bar z)}{\bar z\Phi(\bar z)+\Phi'(\bar z)},
    && (\bar z\equiv -z)
\end{align*}
which implies
\begin{align*}
  \lim_{z\to\infty}\frac{g(z)}{\ell(z)}
  &= \lim_{\bar z\to-\infty}\frac{\Phi'(\bar z)}{\Phi(\bar z)+\bar z\Phi'(\bar z)+\Phi''(\bar z)}
    && \text{(L'H{\^o}pital's rule)} \\
  &= \lim_{\bar z\to-\infty}\frac{\Phi'(\bar z)}{\Phi(\bar z)} \\
  &= \lim_{\bar z\to-\infty}\frac{\Phi''(\bar z)}{\Phi'(\bar z)}
    && \text{(L'H{\^o}pital's rule)} \\
  &= \lim_{\bar z\to-\infty}\frac{-\bar z\Phi'(\bar z)}{\Phi'(\bar z)} \\
  &= \infty.
\end{align*}
Hence, $g(z)$ cannot always be bounded from above by $\ell(z)$, that is, the self-bounding property is not satisfied.

\section{Omitted calculation for examples}
\label{appendix:example}
Here, we compute for each $\phi$,
\[
  \lim_{\mu\downarrow0}\frac{\phi'(\mu)}{\mu\phi''(\mu)}\left[1-\frac{\phi(\mu)}{\mu\phi'(\mu)}\right]
\]
to estimate the power $\alpha$ of the convergence rate provided in \cref{theorem:gd}, by making the error parameter $\bar\epsilon>0$ in \eqref{equation:exponent} arbitrarily small.
Correspondingly, we compute
\[
  \lim_{\mu\downarrow0}\frac{\mu}{[\mu\phi'(\mu)-\phi(\mu)]^\alpha}
\]
to estimate the constant $C_\phi$ in the convergence rate, verifying that $C_\phi$ neither degenerates nor diverges for arbitrarily small error parameter $\bar\epsilon>0$.

Before proceeding with each example, we provide a rough estimate of $\rho$ for loss functions without separation margin.
\begin{lemma}
  \label{lemma:rho_estimate}
  Consider a loss $\ell$ satisfying \cref{assumption:fy_loss,assumption:regular_loss} that does not have separation margin.
  Then,
  \[
    \rho(\lambda) \le -\phi\left(\frac12\right)\lambda.
  \]
\end{lemma}

\begin{proof}
  First, we rewrite $\rho$ as a dual form.
  By introducing the dual variable $\mu$ of $z$ by
  \[
    z=\phi'(\mu) \quad \text{and} \quad \mu=(\phi^*)'(z),
  \]
  we have
  \begin{align*}
    \rho(\lambda) &= \min_{z\in\Rbb}\lambda\ell(z)+z^2 \\
    &= \min_{z\in\Rbb}\lambda\phi^*(z)+z^2 \\
    &= \min_{\mu\in[0,1]}\lambda[\mu\phi'(\mu)-\phi(\mu)]+[\phi'(\mu)]^2,
  \end{align*}
  where we use the definition of the convex conjugate $\phi^*(z)=\mu z-\phi(\mu)$ at the last identity.
  Now, we write the objective as $R(\mu)$:
  \[
    R(\mu) \defeq \lambda[\mu\phi'(\mu)-\phi(\mu)] + [\phi'(\mu)]^2.
  \]
  Differentiating $R$, we have
  \[
    R'(\mu_\star) = [\lambda\mu_\star+2\phi'(\mu_\star)]\phi''(\mu_\star) = 0
    \quad \overset{\phi''>0}{\implies} \quad
    \phi'(\mu_\star)=-\frac\lambda2\mu_\star
  \]
  at the minimizer $\mu_\star$ of $R$.
  Plugging this back to $R$, we have
  \[
    \rho(\lambda)
    = R(\mu_\star)
    = \lambda\left[\mu_\star\left(-\frac\lambda2\mu_\star\right)-\phi(\mu_\star)\right] + \left(-\frac\lambda2\mu_\star\right)^2
    = -\lambda\phi(\mu_\star)
    \le -\phi\left(\frac12\right)\lambda,
  \]
  where the last inequality owes to that a convex potential satisfying \cref{assumption:regularizer} is minimized at the uniform distribution $\mu_\star=1/2$~\cite[Proposition 4]{Blondel2020JMLR}.
\end{proof}

By using \cref{lemma:rho_estimate}, we can simplify the convergence rate of \eqref{equation:gd} given by \cref{theorem:gd} for a loss that does not have separation margin.
Note that the following convergence rate is not sufficiently tight due to overestimation of $\rho$ by \cref{lemma:rho_estimate};
nevertheless, the provided convergence rate is convenient when we do not have an access to $\rho$ analytically.
\begin{corollary}
  \label{corollary:gd_no_separation_margin}
  Under the same setup with \cref{theorem:gd}, we additionally assume that $\ell$ does not have separation margin.
  If $(\alpha,C_\phi)$ with \eqref{equation:exponent} satisfies $\alpha,C_\phi\in(0,\infty)$ and
  \[
    T>\frac{2C_gn}{C_\phi\gamma^2}\epsilon^{-\alpha} + \frac{16[-\phi(1/2)]n^{2}}{C_\phi^2\gamma^2\eta}\epsilon^{-2\alpha}
    \quad \text{for $\epsilon\in(0,\bar\epsilon)$,}
  \]
  then we have $L(\wbf_T)\le\epsilon$.
\end{corollary}

\begin{proof}
  Combining \cref{theorem:gd} and \cref{lemma:rho_estimate}, we have the following convergence rate:
  \[
    T > \frac{4n\sqrt{-\phi(1/2)}\epsilon^{-\alpha}}{C_\phi\gamma\sqrt{\eta}}\sqrt{T} + \frac{C_gn\epsilon^{-\alpha}}{C_\phi\gamma^2}.
  \]
  Defining
  \[
    a\defeq\frac{4n\sqrt{-\phi(1/2)}\epsilon^{-\alpha}}{C_\phi\gamma\sqrt{\eta}}
    \quad \text{and} \quad
    b\defeq\frac{C_gn\epsilon^{-\alpha}}{C_\phi\gamma^2},
  \]
  we have the following inequality in $T$:
  \[
    T^2 - (a^2+2b)T + b^2 > 0.
  \]
  This can be solved for $T\ge1$:
  \[
    T>\frac{a^2+2b}{2}\Biggl[1+\underbrace{\sqrt{1-\left(\frac{2b}{a^2+2b}\right)^2}}_{\le1}\,\Biggr],
  \]
  for which $T>a^2+2b$ is sufficient.
  Thus, we have shown the statement.
\end{proof}

Throughout this section, we repeatedly use L'H{\^o}pital's rule.
When it is used, we notate by $(\ddagger)$.

\subsection{Shannon entropy}
For the Shannon entropy $\phi(\mu)=\mu\ln\mu+(1-\mu)\ln(1-\mu)$, we have
\[
  \phi'(\mu) = \ln\mu - \ln(1-\mu) \quad \text{and} \quad \phi''(\mu) = \frac1\mu + \frac{1}{1-\mu},
\]
which imply
\begin{align*}
  \lim_{\mu\downarrow0}\frac{\phi'(\mu)}{\mu\phi''(\mu)}\left[1-\frac{\phi(\mu)}{\mu\phi'(\mu)}\right]
  &= \lim_{\mu\downarrow0}\frac{\ln\frac{\mu}{1-\mu}}{1-\frac{\mu}{1-\mu}}\left[1-\frac{\mu\ln\mu+(1-\mu)\ln(1-\mu)}{\mu\ln\mu-\mu\ln(1-\mu)}\right] \\
  &= \lim_{\mu\downarrow0}\ln\frac{\mu}{1-\mu}\cdot\frac{\mu\ln\mu-\mu\ln(1-\mu)-\mu\ln\mu-(1-\mu)\ln(1-\mu)}{\mu\ln\frac{\mu}{1-\mu}} \\
  &= \lim_{\mu\downarrow0}\frac{-\ln(1-\mu)}{\mu} \\
  &\overset{(\ddagger)}= \lim_{\mu\downarrow0}\frac{1}{1-\mu} \\
  &= 1,
\end{align*}
and
\begin{align*}
  \lim_{\mu\downarrow0}\frac{\mu}{\mu\phi'(\mu)-\phi(\mu)}
  &= \lim_{\mu\downarrow0}\frac{\mu}{\mu\ln\mu-\mu\ln(1-\mu)-\mu\ln\mu-(1-\mu)\ln(1-\mu)} \\
  &= \lim_{\mu\downarrow0}\frac{\mu}{-\ln(1-\mu)}\cdot\frac{1}{2\mu-1} \\
  &= \lim_{\mu\downarrow0}\frac{\mu}{\ln(1-\mu)} \\
  &\overset{(\ddagger)}= \lim_{\mu\downarrow0}(1-\mu) \\
  &= 1.
\end{align*}

Finally, we derive the convergence rate for the logistic loss.
Plugging $\alpha=1$, $C_\phi=1$, $C_g=1$, and $\rho(\lambda)\le1+\ln^2(\lambda)\le2\ln^2(\lambda)$ to \cref{theorem:gd}, we have
\[
  T > \frac{n}{\gamma^2}\left(\frac{4\sqrt2\ln(\gamma^2\eta T)}{\eta}+1\right)\epsilon^{-1}
  = \left[\frac{4\sqrt2\ln(\gamma^2\eta)}{\eta}+1+\frac{4\sqrt2}{\eta}\ln T\right]\frac{n\epsilon^{-1}}{\gamma^2}.
\]
Dividing both ends by $\ln T$, we have
\[
  \frac{T}{\ln T} > \left[\left(\frac{4\sqrt2\ln(\gamma^2\eta)}{\eta}+1\right)\frac{1}{\ln T}+\frac{4\sqrt2}{\eta}\right]\frac{n\epsilon^{-1}}{\gamma^2},
\]
for which the following is sufficient when $T\ge2$:
\begin{align*}
  \frac{T}{\ln T} &> \left[\left(\frac{4\sqrt2\ln(\gamma^2\eta)}{\eta}+1\right)\frac{1}{\ln2}+\frac{4\sqrt2}{\eta}\right]\frac{n\epsilon^{-1}}{\gamma^2} \\
  &= \left[\frac{4\sqrt2\log_2(\gamma^2\eta)}{\eta}+\frac{1}{\ln2}+\frac{4\sqrt2}{\eta}\right]\frac{n\epsilon^{-1}}{\gamma^2}.
\end{align*}
By ignoring the logarithmic factor, we have
\[
  T\gtrsim\left[\frac{4\sqrt2\log_2(\gamma^2\eta)}{\eta}+\frac{1}{\ln2}+\frac{4\sqrt2}{\eta}\right]\frac{n\epsilon^{-1}}{\gamma^2}.
\]

\subsection{Semi-circle entropy}
For the semi-circle entropy $\phi(\mu)=-2\sqrt{\mu(1-\mu)}$, we first derive the analytical form of the corresponding Fenchel--Young loss.
We have
\[
  \phi'(\mu) = \frac{2\mu-1}{\sqrt{\mu(1-\mu)}} \quad \text{and} \quad \phi''(\mu) = \frac{1}{2[\mu(1-\mu)]^{3/2}}.
\]
The dual transform $(\phi^*)'$ is given by
\[
  (\phi^*)'(z) = (\phi')^{-1}(z) = \frac12\left[\frac{\frac z2}{\sqrt{\left(\frac z2\right)^2+1}}+1\right],
\]
thanks to the Danskin's theorem~\cite{Danskin1966}.
Then, we can derive the Fenchel--Young loss by the definition of the convex conjugate:
\[
  \ell(z) = \phi^*(-z) = -z(\phi^*)'(z)-\phi\left((\phi^*)'(-z)\right) = \frac{-z+\sqrt{z^2+4}}{2}.
\]

Next, we compute the loss parameters $\alpha$ and $C_\phi$ respectively as follows:
\begin{align*}
  \lim_{\mu\downarrow0}\frac{\phi'(\mu)}{\mu\phi''(\mu)}\left[1-\frac{\phi(\mu)}{\mu\phi'(\mu)}\right]
  &= \lim_{\mu\downarrow0}\frac{\frac{2\mu-1}{\sqrt{\mu(1-\mu)}}}{\frac{\mu}{2[\mu(1-\mu)]^{3/2}}}\left[1+\frac{2\sqrt{\mu(1-\mu)}}{\frac{\mu(2\mu-1)}{\sqrt{\mu(1-\mu)}}}\right] \\
  &= \lim_{\mu\downarrow0}2(2\mu-1)(1-\mu)\left[1+\frac{2(1-\mu)}{2\mu-1}\right] \\
  &= 2,
\end{align*}
and
\begin{align*}
  \lim_{\mu\downarrow0}\frac{\mu}{[\mu\phi'(\mu)-\phi(\mu)]^2}
  &= \lim_{\mu\downarrow0}\frac{\mu}{\left[\mu\frac{2\mu-1}{\sqrt{\mu(1-\mu)}}+2\sqrt{\mu(1-\mu)}\right]^2} \\
  &= \lim_{\mu\downarrow0}(1-\mu) \\
  &= 1.
\end{align*}

To estimate $\rho(\lambda)$,
\begin{align*}
  \rho(\lambda) &= \min_{z\in\Rbb}\lambda\ell(z)+z^2 \\
  &\le \lambda\frac{-\ln\lambda+\sqrt{\ln^2\lambda+4}}{2}+\ln^2\lambda && (z=\ln\lambda) \\
  &= \frac{2\lambda}{\ln\lambda+\sqrt{\ln^2\lambda+4}}+\ln^2\lambda \\
  &\le \frac{5\lambda}{2\ln\lambda},
\end{align*}
where we used
\[
  \frac{\lambda}{\ln\lambda} \ge \frac{2\lambda}{\ln\lambda+\sqrt{\ln^2\lambda+4}} \ge \frac23\cdot\ln^2\lambda \quad \text{for $\lambda\ge1$.}
\]

Finally, we derive the convergence rate for the semi-circle loss.
Plugging $\alpha=2$, $C_\phi=1$, $C_g=1$, and $\rho(\lambda)\le5\lambda/(2\ln\lambda)$ to \cref{theorem:gd}, we have
\[
  T>\frac{n}{\gamma^2}\left(\frac{4\sqrt{\frac52\frac{\gamma^2\eta T}{\ln(\gamma^2\eta T)}}}{\eta}+1\right)\epsilon^{-2}
  = \left[\frac{2\sqrt{10}}{\eta}\sqrt{\frac{\gamma^2\eta T}{\ln(\gamma^2\eta T)}}+1\right]\frac{n\epsilon^{-2}}{\gamma^2},
\]
for which the following is sufficient when $T\ge2$:
\[
  T>\left[\frac{2\sqrt{10}}{\eta}\sqrt{\frac{\gamma^2\eta T}{\ln(2\gamma^2\eta)}}+1\right]\frac{n\epsilon^{-2}}{\gamma^2}.
\]
Subsequently, we follow the same flow as in the proof of \cref{corollary:gd_no_separation_margin}.
Defining
\[
  a\defeq\frac{2\sqrt{10}n\epsilon^{-2}}{\gamma^2\eta}\sqrt{\frac{\gamma^2\eta}{\ln(2\gamma^2\eta)}}
  \quad \text{and} \quad
  b\defeq\frac{n\epsilon^{-2}}{\gamma^2},
\]
we have the following inequality in $T$:
\[
  T^2 - (a^2+2b)T + b^2 > 0.
\]
This can be solved for $T\ge1$:
\[
  T>\frac{a^2+2b}{2}\Biggl[1+\underbrace{\sqrt{1-\left(\frac{2b}{a^2+2b}\right)^2}}_{\le1}\Biggr],
\]
for which $T>a^2+2b$ is sufficient, namely,
\[
  T>\frac{40n^2}{\gamma^2\eta\ln(2\gamma^2\eta)}\epsilon^{-4} + \frac{2n}{\gamma^2}\epsilon^{-2}
\]
is sufficient.
Thus, the convergence rate is $T=\Omega(\epsilon^{-4})$.

\subsection{Tsallis entropy}
For the Tsallis entropy
\[
  \phi(\mu)=\frac{\mu^q+(1-\mu)^q-1}{q-1},
\]
define
\begin{align*}
  \phi_0(\mu) &= \mu^q+(1-\mu)^q-1, \\
  \phi_1(\mu) &= \mu^{q-1}-(1-\mu)^{q-1}, \\
  \phi_2(\mu) &= \mu^{q-2}+(1-\mu)^{q-2}.
\end{align*}
When $0<q<2$ ($q\ne 1$),
\begin{align*}
  \lim_{\mu\downarrow0}\frac{\phi'(\mu)}{\mu\phi''(\mu)}\left[1-\frac{\phi(\mu)}{\mu\phi'(\mu)}\right]
  &= \frac{1}{q(q-1)}\lim_{\mu\downarrow0}\frac{1}{\phi_2(\mu)} \cdot \frac{q\mu\phi_1(\mu)-\phi_0(\mu)}{\mu^2} \\
  &= \frac{1}{q(q-1)}\lim_{\mu\downarrow0}\frac{1}{1+\left(\frac{\mu}{1-\mu}\right)^{2-q}} \cdot \frac{q\mu\phi_1(\mu)-\phi_0(\mu)}{\mu^q} \\
  &= \frac{1}{q(q-1)}\lim_{\mu\downarrow0}\frac{q\mu\phi_1(\mu)-\phi_0(\mu)}{\mu^q} \\
  &\overset{(\ddagger)}= \frac{1}{q(q-1)}\lim_{\mu\downarrow0}\frac{q\phi_1(\mu)+q(q-1)\mu\phi_2(\mu)-q\phi_1(\mu)}{q\mu^{q-1}} \\
  &= \frac{1}{q}\lim_{\mu\downarrow0}\frac{\mu^{q-1}+(1-\mu)^{q-2}\mu}{\mu^{q-1}} \\
  &\overset{(\ddagger)}= \frac{1}{q}\lim_{\mu\downarrow0}\frac{(q-1)\mu^{q-2}+(1-\mu)^{q-2}-(q-2)(1-\mu)^{q-3}\mu}{(q-1)\mu^{q-2}} \\
  &= \frac{1}{q}\lim_{\mu\downarrow0}\left[1+\frac{1}{q-1}\left(\frac{\mu}{1-\mu}\right)^{2-q}-(q-2)\left(\frac{\mu}{1-\mu}\right)^{3-q}\right] \\
  &= \frac1q.
\end{align*}
In addition, we have
\begin{align*}
  &\lim_{\mu\downarrow0}\frac{\mu}{[\mu\phi'(\mu)-\phi(\mu)]^{1/q}} \\
  &= \lim_{\mu\downarrow0}\frac{(q-1)^{1/q}\mu}{[q\mu\phi_1(\mu)-\phi_0(\mu)]^{1/q}} \\
  &= (q-1)^{1/q}\left\{\lim_{\mu\downarrow0}\frac{q\mu\phi_1(\mu)-\phi_0(\mu)}{\mu^q}\right\}^{-1/q} \\
  &= (q-1)^{1/q}\left\{q-1-\lim_{\mu\downarrow0}\frac{q\mu(1-\mu)^{q-1}+(1-\mu)^q-1}{\mu^q}\right\}^{-1/q} \\
  &\overset{(\ddagger)}= (q-1)^{1/q}\left\{q-1-\lim_{\mu\downarrow0}\frac{q(1-\mu)^{q-1}-q(q-1)\mu(1-\mu)^{q-2}-q(1-\mu)^{q-1}}{q\mu^{q-1}}\right\}^{-1/q} \\
  &= (q-1)^{1/q}\left\{q-1-(q-1)\lim_{\mu\downarrow0}\left(\frac{\mu}{1-\mu}\right)^{2-q}\right\}^{-1/q} \\
  &= (q-1)^{1/q}\cdot(q-1+0)^{-1/q} \\
  &= 1.
\end{align*}

When $q\ge2$,
\begin{align*}
  &\lim_{\mu\downarrow0}\frac{\phi'(\mu)}{\mu\phi''(\mu)}\left[1-\frac{\phi(\mu)}{\mu\phi'(\mu)}\right] \\
  &= \frac{1}{q(q-1)}\lim_{\mu\downarrow0}\frac{1}{\phi_2(\mu)} \cdot \frac{q\mu\phi_1(\mu)-\phi_0(\mu)}{\mu^2} \\
  &= \frac{1}{q(q-1)}\lim_{\mu\downarrow0}\frac{q\mu\phi_1(\mu)-\phi_0(\mu)}{\mu^2} \\
  %&= \frac{1}{q(q-1)}\lim_{\mu\downarrow0}\frac{1-q\mu(1-\mu)^{q-1}-(1-\mu)^q}{\mu^2} \\
  &= \frac{1}{q(q-1)}\lim_{\mu\downarrow0}\frac{q[\mu^q-(1-\mu)^{q-1}\mu]-\mu^q-(1-\mu)^q}{\mu^2} \\
  &\overset{(\ddagger)}= \frac{1}{q(q-1)}\lim_{\mu\downarrow0}\frac{q[q\mu^{q-1}-(1-\mu)^{q-1}+(q-1)(1-\mu)^{q-2}\mu]-q\mu^{q-1}+q(1-\mu)^{q-1}}{2\mu} \\
  &= \lim_{\mu\downarrow0}\frac{\mu^{q-2}+(1-\mu)^{q-2}}{2} \\
  &= \frac12.
\end{align*}
In addition, we have
\begin{align*}
  \lim_{\mu\downarrow0}\frac{\mu}{[\mu\phi'(\mu)-\phi(\mu)]^{1/2}}
  &= (q-1)^{1/2}\left\{\lim_{\mu\downarrow0}\frac{q\mu\phi_1(\mu)-\phi_0(\mu)}{\mu^2}\right\}^{-1/2} \\
  &\overset{(\ddagger)}= (q-1)^{1/2}\left\{\lim_{\mu\downarrow0}\frac{q\phi_1(\mu)+q(q-1)\mu\phi_2(\mu)-q\phi_1(\mu)}{2\mu}\right\}^{-1/2} \\
  &= (q-1)^{1/2}\left\{\frac{q(q-1)}{2}\lim_{\mu\downarrow0}[\mu^{q-2}+(1-\mu)^{q-2}]\right\}^{-1/2} \\
  &= (q-1)^{1/2}\cdot\left[\frac{q(q-1)}{2}\right]^{-1/2} \\
  &= \sqrt{\frac2q}.
\end{align*}

\subsection{R{\'e}nyi entropy}
For the R{\'e}nyi entropy
\[
  \phi(\mu)=\frac{1}{q-1}\ln\left[\mu^q+(1-\mu)^q\right],
\]
define
\begin{align*}
  \phi_0(\mu) &= \mu^q+(1-\mu)^q, \\
  \phi_1(\mu) &= \mu^{q-1}-(1-\mu)^{q-1}, \\
  \phi_2(\mu) &= \mu^{q-2}+(1-\mu)^{q-2}, \\
  \phi_3(\mu) &= \mu^{q-3}-(1-\mu)^{q-3}.
\end{align*}
When $0<q<2$ with $q\ne1$,
\begin{align*}
  &\lim_{\mu\downarrow0}\frac{\phi'(\mu)}{\mu\phi''(\mu)}\left[1-\frac{\phi(\mu)}{\mu\phi'(\mu)}\right] \\
  &= \lim_{\mu\downarrow0}\frac{\frac{\phi_1(\mu)}{\phi_0(\mu)}}{(q-1)\mu\frac{\phi_2(\mu)}{\phi_0(\mu)}-q\mu\frac{\phi_1(\mu)^2}{\phi_0(\mu)^2}}\left[1-\frac{\frac{1}{q-1}\ln\phi_0(\mu)}{\frac{q}{q-1}\mu\frac{\phi_1(\mu)}{\phi_0(\mu)}}\right] \\
  &= \lim_{\mu\downarrow0}\frac{1}{(q-1)\frac{\mu\phi_2(\mu)}{\phi_1(\mu)}-q\frac{\mu\phi_1(\mu)}{\phi_0(\mu)}} \cdot \left[1-\frac{\phi_0(\mu)\ln\phi_0(\mu)}{q\mu\phi_1(\mu)}\right] \\
  &= \frac{1}{(q-1)\lim_{\mu\downarrow0}\frac{\mu\phi_2(\mu)}{\phi_1(\mu)}-q\lim_{\mu\downarrow0}\frac{\mu\phi_1(\mu)}{\phi_0(\mu)}} \cdot \left[1-\frac{\lim_{\mu\downarrow0}\phi_0(\mu)}{q}\cdot\lim_{\mu\downarrow0}\frac{\ln\phi_0(\mu)}{\mu\phi_1(\mu)}\right] \\
  &= \frac{1}{(q-1)\cdot1-q\cdot0}\cdot\left[1-\frac{1}{q}\cdot1\right] \\
  &= \frac1q,
\end{align*}
where we use
\[
  \phi_0(\mu) \to 1, \quad
  \frac{\mu\phi_2(\mu)}{\phi_1(\mu)} = \frac{1+\left(\frac{\mu}{1-\mu}\right)^{2-q}}{1-\left(\frac{\mu}{1-\mu}\right)^{2-q}} \to 1, \quad
  \mu\phi_1(\mu) = \mu^q-\frac{\mu}{(1-\mu)^{1-q}} \to 0,
\]
and
\begin{align*}
  \frac{\ln\phi_0(\mu)}{\mu\phi_1(\mu)}
  &\overset{(\ddagger)}\to \frac{1}{\phi_0(\mu)}\cdot\frac{\phi_0'(\mu)}{\mu\phi_1'(\mu)+\phi_1(\mu)} \\
  &\to \frac{\phi_0'(\mu)}{\mu\phi_1'(\mu)+\phi_1(\mu)} \\
  &= \frac{q\phi_1(\mu)}{(q-1)\mu\phi_2(\mu)+\phi_1(\mu)} \\
  &= \frac{q}{(q-1)\frac{\mu\phi_2(\mu)}{\phi_1(\mu)}+1} \\
  &\to \frac{q}{(q-1)\cdot1+1} \\
  &= 1.
\end{align*}
In addition, we have
\begin{align*}
  &\lim_{\mu\downarrow0}\frac{\mu}{[\mu\phi'(\mu)-\phi(\mu)]^{1/q}} \\
  &= \left\{\lim_{\mu\downarrow0}\frac{\mu^q}{\mu\phi'(\mu)-\phi(\mu)}\right\}^{1/q} \\
  &= \left\{\lim_{\mu\downarrow0}\frac{(q-1)\mu^q\phi_0(\mu)}{q\mu\phi_1(\mu)-\phi_0(\mu)\ln\phi_0(\mu)}\right\}^{1/q} \\
  &= \left\{\lim_{\mu\downarrow0}\frac{(q-1)\mu^q}{q\mu\phi_1(\mu)-\phi_0(\mu)\ln\phi_0(\mu)}\right\}^{1/q}
    \qquad\qquad (\phi_0(\mu)\to1) \\
  &\overset{(\ddagger)}= \left\{(q-1)\lim_{\mu\downarrow0}\frac{q\mu^{q-1}}{q\phi_1(\mu)+q(q-1)\mu\phi_2(\mu)-q\phi_1(\mu)\ln\phi_0(\mu)-q\phi_1(\mu)}\right\}^{1/q} \\
  &= \left\{(q-1)\lim_{\mu\downarrow0}\frac{\mu^{q-1}}{(q-1)\mu\phi_2(\mu)-\phi_1(\mu)\ln\phi_0(\mu)}\right\}^{1/q} \\
  &\overset{(\ddagger)}= \left\{(q-1)\lim_{\mu\downarrow0}\frac{(q-1)\mu^{q-2}}{(q-1)\phi_2(\mu)+(q-1)(q-2)\mu\phi_3(\mu)-\frac{q\phi_1(\mu)^2}{\phi_0(\mu)}-(q-1)\phi_2(\mu)\ln\phi_0(\mu)}\right\}^{1/q} \\
  &= \lim_{\mu\downarrow0}\!\bigg\{\!\frac{\big[1+\big(\frac{\mu}{1-\mu}\big)^{2-q}\big] \!\! + \! (q-2)\big[1-\big(\frac{\mu}{1-\mu}\big)^{3-q}\big] \!\! - \! \frac{q\phi_1(\mu)^2}{(q-1)\mu^{q-2}\phi_0(\mu)} \! - \! \big[1+\big(\frac{\mu}{1-\mu}\big)^{2-q}\big]\ln\phi_0(\mu)}{q-1}\!\bigg\}^{\!\!-\frac1q} \\
  &\overset{\text{(A)}}= \left\{\frac{1+(q-2)\cdot1-0-1\cdot0}{q-1}\right\}^{-1/q} \\
  &= 1,
\end{align*}
where at (A) we used
\begin{align*}
  \frac{\phi_1(\mu)^2}{\mu^{q-2}\phi_0(\mu)}
  &\to \mu^{2-q}\phi_1(\mu)^2 \\
  &= \mu^q-2(1-\mu)^{q-1}\mu+(1-\mu)^{2q-2}\mu^{2-q} \\
  &\to -2(1-\mu)^{q-1}\mu+(1-\mu)^{2q-2}\mu^{2-q} \\
  &= \frac{(1-\mu)^{2q-2}-2(1-\mu)^{q-1}\mu^{q-1}}{\mu^{q-2}} \\
  &\overset{(\ddagger)}\to \frac{(2q-2)(1-\mu)^{2q-3}+2(q-1)(1-\mu)^{q-2}\mu^{q-1}-2(q-1)(1-\mu)^{q-1}\mu^{q-2}}{(q-2)\mu^{q-3}} \\
  &= \frac{2(q-1)}{q-2}\cdot\frac{1}{(1-\mu)^{2-q}}\cdot\frac{(1-\mu)^{q-1}+\mu^{q-1}-(1-\mu)\mu^{q-2}}{\mu^{q-3}} \\
  &\to \frac{2(q-1)}{q-2}\cdot1\cdot\frac{(1-\mu)^{q-1}+\mu^{q-1}-(1-\mu)\mu^{q-2}}{\mu^{q-3}} \\
  &= \frac{2(q-1)}{q-2}\cdot\left[\left(\frac{\mu}{1-\mu}\right)^{1-q}\mu^2+\mu^2-(1-\mu)\mu\right] \\
  &\to 0.
\end{align*}

When $q=2$, we leverage
\[
  \phi_0(\mu)=2\mu^2-2\mu+1, \quad \phi_1(\mu)=2\mu-1, \quad \phi_2(\mu)=2, \quad
  \phi_0'(\mu)=2\phi_1(\mu), \quad \phi_1'(\mu)=2
\]
to have
\begin{align*}
  &\lim_{\mu\downarrow0}\frac{\phi'(\mu)}{\mu\phi''(\mu)}\left[1-\frac{\phi(\mu)}{\mu\phi'(\mu)}\right] \\
  &= \lim_{\mu\downarrow0}\phi_0(\mu)\cdot\frac{2\mu\phi_1(\mu)-\phi_0(\mu)\ln\phi_0(\mu)}{4\mu^2[\phi_0(\mu)-\phi_1(\mu)^2]} \\
  &= \lim_{\mu\downarrow0}\frac{2\mu\phi_1(\mu)-\phi_0(\mu)\ln\phi_0(\mu)}{4\mu^2[\phi_0(\mu)-\phi_1(\mu)^2]}
    && (\phi_0(\mu)\to1) \\
  &\overset{(\ddagger)}= \lim_{\mu\downarrow0}\frac{2[\phi_1(\mu)+2\mu]-2\phi_1(\mu)\ln\phi_0(\mu)-2\phi_1(\mu)}{4\{2\mu[\phi_0(\mu)-\phi_1(\mu)^2]+\mu^2[2\phi_1(\mu)-4\phi_1(\mu)]\}} \\
  &= \lim_{\mu\downarrow0}\frac{2\mu-\phi_1(\mu)\ln\phi_0(\mu)}{4\mu[\phi_0(\mu)-\phi_1(\mu)^2-\mu\phi_1(\mu)]} \\
  &\overset{(\ddagger)}= \lim_{\mu\downarrow0}\frac{2-2\ln\phi_0(\mu)-\frac{2\phi_1(\mu)^2}{\phi_0(\mu)}}{4[\phi_0(\mu)-\phi_1(\mu)^2-\mu\phi_1(\mu)]+4\mu[2\phi_1(\mu)-4\phi_1(\mu)-\phi_1(\mu)-2\mu]} \\
  &= \lim_{\mu\downarrow0}\frac{1-\ln\phi_0(\mu)-\frac{\phi_1(\mu)^2}{\phi_0(\mu)}}{2[\phi_0(\mu)-\phi_1(\mu)^2-4\mu\phi_1(\mu)-2\mu^2]} \\
  &\overset{(\ddagger)}= \lim_{\mu\downarrow0}\frac{-\frac{2\phi_1(\mu)}{\phi_0(\mu)}-\frac{4\phi_0(\mu)\phi_1(\mu)-2\phi_1(\mu)^3}{\phi_0(\mu)^2}}{2[2\phi_1(\mu)-4\phi_1(\mu)-4\phi_1(\mu)-8\mu-4\mu]} \\
  &= \lim_{\mu\downarrow0}\frac{\phi_1(\mu)}{\phi_0(\mu)^2}\cdot\frac{3\phi_0(\mu)-\phi_1(\mu)^2}{6[\phi_1(\mu)+2\mu]} \\
  &=\frac13.
\end{align*}
In addition, defining
\[
  \zeta\defeq\frac{2\mu\phi_1(\mu)-\phi_0(\mu)\ln\phi_0(\mu)}{\phi_0(\mu)},
\]
we have
\begin{align*}
  (\xi&\defeq) \lim_{\mu\downarrow0}\frac{\mu}{[\mu\phi'(\mu)-\phi(\mu)]^{1/3}} \\
  &= \lim_{\mu\downarrow0}\frac{\mu}{\left[\frac{2\mu\phi_1(\mu)-\phi_0(\mu)\ln\phi_0(\mu)}{\phi_0(\mu)}\right]^{1/3}}
    \quad \left(\text{implies~~} \xi=\lim_{\mu\downarrow0}\mu\zeta^{-1/3} \text{;~we will use this below at (\$)}\right) \\
  &\overset{(\ddagger)}= \lim_{\mu\downarrow0}\frac{3\zeta^{2/3}}{\frac{2\phi_1(\mu)}{\phi_0(\mu)}+2\mu\frac{2\phi_0(\mu)-2\phi_1(\mu)^2}{\phi_0(\mu)^2}-\frac{2\phi_1(\mu)}{\phi_0(\mu)}} \\
  &= \lim_{\mu\downarrow0}\frac{3\phi_0(\mu)^2}{4}\frac{\zeta^{2/3}}{\mu[\phi_0(\mu)-\phi_1(\mu)^2]} \\
  &\overset{(\ddagger)}= \lim_{\mu\downarrow0}\frac{3}{4}\frac{\frac23\cdot2\mu\frac{2\phi_0(\mu)-2\phi_1(\mu)^2}{\phi_0(\mu)^2}}{\left\{\phi_0(\mu)-\phi_1(\mu)^2+\mu[2\phi_1(\mu)-4\phi_1(\mu)]\right\}\zeta^{1/3}} \\
  &= \lim_{\mu\downarrow0}\frac{2\mu[\phi_0(\mu)-\phi_1(\mu)^2]}{[\phi_0(\mu)-\phi_1(\mu)^2-2\mu\phi_1(\mu)]\zeta^{1/3}} \\
  &\overset{(\ddagger)}= \lim_{\mu\downarrow0}\frac{3\zeta^{2/3}}{2}\frac{1}{\mu\phi_0(\mu)[\phi_0(\mu)-\phi_1(\mu)^2]-3\frac{\phi_0(\mu)^3[\phi_1(\mu)+\mu]\zeta}{\phi_0(\mu)-\phi_1(\mu)^2-2\mu\phi_1(\mu)}} \\
  &= \lim_{\mu\downarrow0}\frac{3\zeta^{2/3}}{2}\frac{1}{\mu[\phi_0(\mu)-\phi_1(\mu)^2]+\frac{3\zeta}{\phi_0(\mu)-\phi_1(\mu)^2-2\mu\phi_1(\mu)}} \\
  &= \left\{\lim_{\mu\downarrow0}\frac23\frac{\phi_0(\mu)-\phi_1(\mu)^2}{\mu}\cdot(\mu\zeta^{-1/3})^2 + \lim_{\mu\downarrow0}\frac{2\mu}{\phi_0(\mu)-\phi_1(\mu)^2-2\mu\phi_1(\mu)}\frac{1}{\mu\zeta^{-1/3}}\right\}^{-1} \\
  &\overset{\text{(\$)}}= \left\{\frac23\xi^2\lim_{\mu\downarrow0}(2-2\mu) + \frac1\xi\lim_{\mu\downarrow0}\frac{1}{2-3\mu}\right\}^{-1} \\
  &= \left\{\frac43\xi^2+\frac{1}{2\xi}\right\}^{-1},
\end{align*}
which implies
\[
  \xi = \frac{1}{\frac43\xi^2+\frac{1}{2\xi}}.
\]
By solving this, we have
\[
  \lim_{\mu\downarrow0}\frac{\mu}{[\mu\phi'(\mu)-\phi(\mu)]^{1/3}} = \xi = \left(\frac38\right)^{1/3}.
\]

\begin{figure}
  \centering
  \includegraphics[width=0.75\textwidth]{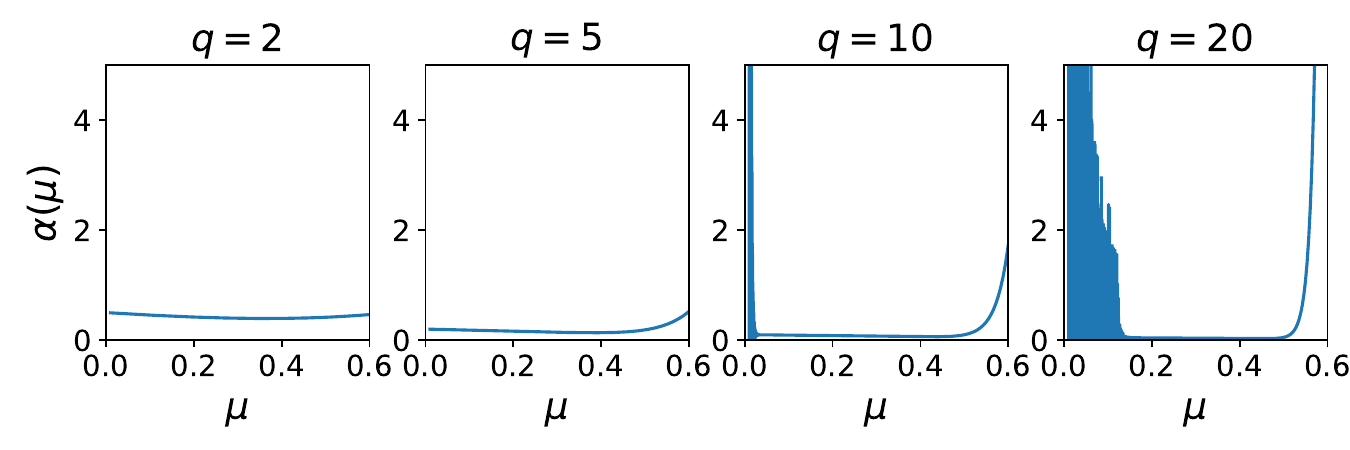}
  \caption{
    For the pseudo-spherical entropy, $\alpha(\mu)=[\phi'(\mu)/\mu\phi''(\mu)] \cdot [1-\phi(\mu)/\mu\phi'(\mu)]$ is shown.
  }
  \label{figure:norm_ent}
\end{figure}

\subsection{Pseudo-spherical entropy}
Consider the $q$-pseudo-spherical entropy $\phi(\mu)=[\mu^q+(1-\mu)^q]^{1/q}-1$ for $q>1$~\cite{Gneiting2007}. It is also known as the $q$-norm (neg)entropy~\cite{Boekee1980}.
When $q=2$, it recovers the spherical entropy associated with the spherical loss~\cite{Agarwal2014JMLR}.
When $q\uparrow\infty$, it approaches $\phi_\infty(\mu)=\max\set{\mu,1-\mu}-1$, which is the Bayes risk of the hinge/0-1 losses~\cite{Buja2005}.
As seen in \cref{figure:norm_ent}, the limit $\alpha$ (in \eqref{equation:exponent}) does not exist, which indicates that we cannot guarantee the $\epsilon$-optimal risk for vanishingly small $\epsilon$.

\end{document}

%% file: preamble.tex
\usepackage{amsmath}
\usepackage{amssymb}
\usepackage{amsthm}
\usepackage{bm}
\usepackage{booktabs}
\usepackage[font=footnotesize,labelfont=bf]{caption}
\usepackage[capitalize]{cleveref}
\usepackage{color}
\usepackage{dsfont}
\usepackage{enumitem}
\usepackage[T1]{fontenc}
\usepackage{mathtools}
\usepackage{mathrsfs}
\usepackage{multirow}
\usepackage{pifont}
\usepackage{rotating}
\usepackage[group-separator={,},group-minimum-digits={3}]{siunitx}
\usepackage{subcaption}
\usepackage{tikz}
\usepackage{thm-restate}
\usepackage{wrapfig}

\setitemize{itemsep=-1pt,topsep=3pt}
\renewcommand{\tilde}{\widetilde}

\newcommand{\Ccal}{\mathcal{C}}

\newcommand{\Fcal}{\mathcal{F}}

\newcommand{\Ical}{\mathcal{I}}

\newcommand{\Ocal}{\mathcal{O}}

\newcommand{\Scal}{\mathcal{S}}

\newcommand{\Ebb}{\mathbb{E}}

\newcommand{\Nbb}{\mathbb{N}}

\newcommand{\Rbb}{\mathbb{R}}

\newcommand{\Pbf}{\mathbf{P}}

\newcommand{\ebf}{\mathbf{e}}

\newcommand{\ubf}{\mathbf{u}}

\newcommand{\wbf}{\mathbf{w}}
\newcommand{\xbf}{\mathbf{x}}

\newcommand{\zbf}{\mathbf{z}}

\newcommand{\thetabf}{\bm{\theta}}

\newcommand{\mubf}{\bm{\mu}}

\newcommand{\zerobf}{\mathbf{0}}
\newcommand{\onebf}{\mathbf{1}}

\newcommand{\defeq}{\coloneqq}
\newcommand{\eqdef}{\eqqcolon}

\DeclareMathOperator{\domain}{\mathrm{dom}}
\DeclareMathOperator{\image}{\mathrm{Im}}

\DeclareMathOperator*{\E}{\Ebb}

\DeclareMathOperator{\interior}{\mathrm{int}}

\newcommand{\set}[1]{\left\lbrace{#1}\right\rbrace}
\newcommand{\setcomp}[2]{\left\lbrace{#1} \relmiddle| {#2}\right\rbrace}
\newcommand{\inpr}[2]{\left\langle{#1},{#2}\right\rangle}
\newcommand{\relmiddle}[1]{\mathrel{}\middle#1\mathrel{}}

\newcommand{\rd}{\mathrm{d}}

\renewcommand{\epsilon}{\varepsilon}

\newtheorem{theorem}{Theorem}
\newtheorem{definition}[theorem]{Definition}
\newtheorem{assumption}{Assumption}
\newtheorem{lemma}[theorem]{Lemma}
\newtheorem{proposition}[theorem]{Proposition}
\newtheorem{corollary}[theorem]{Corollary}